 \documentclass[accepted]{uai2021} 

\usepackage[american]{babel}
\usepackage{microtype}
\usepackage{graphicx}
\usepackage{subfigure}
\usepackage{booktabs} 

\usepackage{hyperref}
\usepackage{url}
\usepackage{amssymb}
\usepackage{algorithm}
\usepackage{algorithmic}
\usepackage{mathtools}
\usepackage{tikz}
\usetikzlibrary{shapes,arrows}
\usepackage{xcolor}
\usepackage{nicefrac}
\usepackage{subfigure}
\usepackage{amsmath}
\usetikzlibrary{matrix}
\usetikzlibrary{backgrounds,automata}
\usetikzlibrary{arrows, positioning, automata}
\usepackage[edges]{forest}
\usetikzlibrary{shapes.geometric,arrows.meta}
\usepackage{subfigure}
\usepackage{hyperref}
\usepackage{url}
\usepackage{xcolor}
\usepackage{nicefrac}
\usetikzlibrary{backgrounds,automata}
\usetikzlibrary{shadows,matrix}
\usetikzlibrary{shapes,arrows}
\usepackage{amsmath}
\usepackage{textcomp}
\usepackage[utf8]{inputenc}
\usepackage{phaistos}
\usetikzlibrary{arrows, positioning, automata}
\usepackage[edges]{forest}
\usetikzlibrary{shapes.geometric,arrows.meta}
\usepackage{subfigure}
\usepackage{hyperref}
\usepackage{url}
\usepackage{xcolor}
\usepackage{nicefrac}
\usepackage{amsmath}
\usepackage{textcomp}
\usepackage{phaistos}
\usepackage{enumitem}
\usetikzlibrary{positioning}

\usepackage{amsthm}
\newtheorem{theorem}{Theorem}
\theoremstyle{definition}
\newtheorem{definition}{Definition}

\DeclareMathOperator*{\argmax}{arg\,max}

\DeclareMathAlphabet{\mathcal}{OMS}{cmsy}{m}{n}

\hypersetup{colorlinks,citecolor=blue, linkcolor=blue,urlcolor=blue}

\usepackage{natbib} 
\usepackage{mathtools} 
\usepackage{booktabs} 
\usepackage{tikz} 

\title{Variational Combinatorial Sequential Monte Carlo Methods for Bayesian Phylogenetic Inference}

\author[1,*]{
{Antonio Khalil Moretti}{}} 
\author[1,*]{Liyi Zhang}
\author[1]{Christian A. Naesseth}
\author[1]{Hadiah Venner}
\author[1]{David Blei}
\author[1]{Itsik Pe'er}
\affil[ ]{%
    Columbia University 
}
\affil[ ]{%
$\{$amoretti, itsik$\}$@cs.columbia.edu\\
$\{$lz2574, christian.a.naesseth, hkv2001, david.blei$\}$@columbia.edu
}

\begin{document}
\maketitle

\begin{abstract}
Bayesian phylogenetic inference is often conducted via local or sequential search over topologies and branch lengths using algorithms such as random-walk Markov chain Monte Carlo (\textsc{Mcmc}) or Combinatorial Sequential Monte Carlo (\textsc{Csmc}). 
However, when \textsc{Mcmc} is used for evolutionary parameter learning, convergence requires long runs with inefficient exploration of the state space. We introduce Variational Combinatorial Sequential Monte Carlo (\textsc{Vcsmc}), 
a powerful framework that establishes 
variational sequential search to learn distributions over intricate combinatorial structures.
We then develop nested \textsc{Csmc}, an efficient proposal distribution for \textsc{Csmc} and prove that nested \textsc{Csmc} is an exact approximation to the (intractable) locally optimal proposal. We use nested \textsc{Csmc} to define a second objective, \textsc{Vncsmc} which yields tighter lower bounds than \textsc{Vcsmc}. 
We show that \textsc{Vcsmc} and \textsc{Vncsmc} are computationally efficient and 
explore higher probability spaces than existing methods on a range of tasks.
\end{abstract}

\section{Introduction}\label{sec:intro}
\let\thefootnote\relax\footnotetext{* = Authors contributed equally}
What is the origin of \textsc{Sars-CoV-II} and how can we analyze the progression of its genetic variants? How do antibodies evolve and develop in response to infection and vaccination? Bayesian phylogenetic inference is a powerful statistical tool to address these and other questions of central importance in molecular evolutionary biology and epidemiology~\citep{Dhar_2020,abf4c405af0c4229bee83cefc6b9501f}. Given an evolutionary model and an alignment of observed molecular sequences (\textsc{Dna}, \textsc{Rna}, \textsc{Protein}), Bayesian methods sample latent bifurcating trees to uncover genetic history, quantify uncertainty and incorporate prior information~\citep{10.1093/bioinformatics/17.8.754}.
Phylogenetic modeling involves three distinct challenges: ($i$) sampling from a discrete distribution to approximate an intractable summation over tree topologies, ($ii$) for each tree, integrating over the continuous branch lengths that govern the stochastic process for genetic mutations, and ($iii$) performing parameter optimization or model learning. The \textit{marginalization} of tree topologies and branch lengths is typically accomplished via local search algorithms such as random-walk Markov chain Monte Carlo (\textsc{Mcmc})~\citep{10.1093/bioinformatics/17.8.754} or sequential search algorithms such as Combinatorial Sequential Monte Carlo (\textsc{Csmc})~\citep{[pset]}. 
Sophisticated proposal methods based on Hamiltonian Monte Carlo or particle \textsc{Mcmc} have been suggested to simultaneously sample from composite spaces and optimize evolutionary parameters~\citep{pmlr-v70-dinh17a,csmc,wang2020particle}. However, these methods are often difficult to implement, slow to converge requiring days or weeks of CPU time, and heavily dependent upon heuristics.

Variational Inference (\textsc{Vi}) is a computationally efficient alternative to \textsc{Mcmc}. 
\textsc{Vi} posits an approximate posterior and then recovers parameters of both the model and approximate posterior by maximizing a lower bound to the log-marginal likelihood. One approach to learning variational distributions on phylogenetic trees is to parameterize the tree as a sequence of \textit{subsplits}, or ordered partitions on clades, and to recast the problem as a Bayesian network~\citep{NIPS2018_7418}. The drawback of this setup is that the support of the conditional probability tables scales exponentially with the number of taxa~\citep{zhang2018variational}. A body of recent work has established connections between \textsc{Vi} and sequential search by defining a variational family of distributions on hidden Markov models, where Sequential Monte Carlo (\textsc{Smc}) is used as the marginal likelihood estimator~\citep{maddison2017filtering,anh2018autoencoding,pmlr-v84-naesseth18a,lawson2018twisted,moretti2019smoothing,moretti2019particle,naesseth2020msc,MLCB,moretti2020psvo,Moretti2021}. We extend these approaches by developing variational sequential search methods that learn distributions over complex combinatorial structures.
Our contributions are as follows:
\begin{itemize}
    \item We develop Variational Combinatorial Sequential Monte Carlo (\textsc{Vcsmc}), a novel variational objective and structured approximate posterior defined on the space of phylogenetic trees. \textsc{Vcsmc} blends \textsc{Csmc} and \textsc{Vi}, providing the user with a flexible and powerful approximate inference algorithm.
    \item We further extend \textsc{Csmc} with nested \textsc{Smc} \citep{naesseth2015nested,naesseth2016high}, introducing a new efficient proposal distribution for \textsc{Csmc}.
We prove that this proposal is an \emph{exact approximation} to the (intractable) locally optimal proposal for \textsc{Csmc}. We use \textsc{Ncsmc} to define a second objective, \textsc{Vncsmc} which yields tighter lower bounds than \textsc{Vcsmc}.
\item In empirical studies, we demonstrate the advantage of \textsc{Vcsmc} and \textsc{Vncsmc}. First, we analyze 
a standard dataset of primate mitochondrial DNA, then the complete genomes of 17 Betacoronavirus species over 36,889 sites, and finally 7 benchmark datasets (DS1-DS7) ranging from 27 to 64 taxa. 
\textsc{Vcsmc} and \textsc{Vncsmc} are compared to existing benchmarks 
and shown to perform favorably across a range of tasks.
\end{itemize}


\paragraph{Related Work.}
Bayesian phylogenetics is often approximated using local search algorithms such as random-walk \textsc{Mcmc}~\citep{10.1093/bioinformatics/17.8.754} or sequential search algorithms such as \textsc{Csmc}~\citep{[pset]}. 
\textsc{Mcmc} methods can also be used for model learning, jointly estimating the phylogenetic trees and evolutionary parameters. Probabilistic path Hamiltonian Monte Carlo ($ppHMC$)~\citep{pmlr-v70-dinh17a} is one such method that extends Hamiltonian Monte Carlo by defining a Markov chain on the orthant complex of phylogenetic tree space. 
It is often the case that the likelihood term in the \textsc{Mcmc} acceptance ratio is difficult to evaluate. The idea of Particle \textsc{Mcmc} algorithms (\textsc{Pmcmc}) is to use $\textsc{Smc}$ as an unbiased estimate of the marginal likelihood 
to define a proposal for $\textsc{Mcmc}$~\citep{doi:10.1111/j.1467-9868.2009.00736.x}. A \textsc{Pmcmc} algorithm for evolutionary parameter learning was introduced in~\citep{csmc}, and improved upon using a particle Gibbs sampler in \citep{wang2020particle}. 
In contrast to these methods, the proposed approach leverages \textsc{Vi} for inference and introduces a new efficient proposal distribution for \textsc{Csmc}.

One approach to \textsc{Vi} for phylogenetic trees is to parameterize a tree as a sequence of \textit{subsplits}, or ordered partitions on clades and to recast the problem as a Bayesian network~\citep{NIPS2018_7418}. A drawback of this setup is that the support of the conditional probability tables scales exponentially with the number of taxa~\citep{zhang2018variational}. In subsequent work, the authors introduce two Variational Bayesian Phylogenetic Inference frameworks (\textsc{Vbpi} and \textsc{Vbpi-Nf}) by using pre-computed topologies to define the support of the conditional probability tables for the approximation~\citep{zhang2018variational,zhang2020improved}. In contrast, \textsc{Vcsmc} does not restrict the support of the tree topologies and instead leverages \textsc{Csmc} to compute a lower bound.

\section{Background}

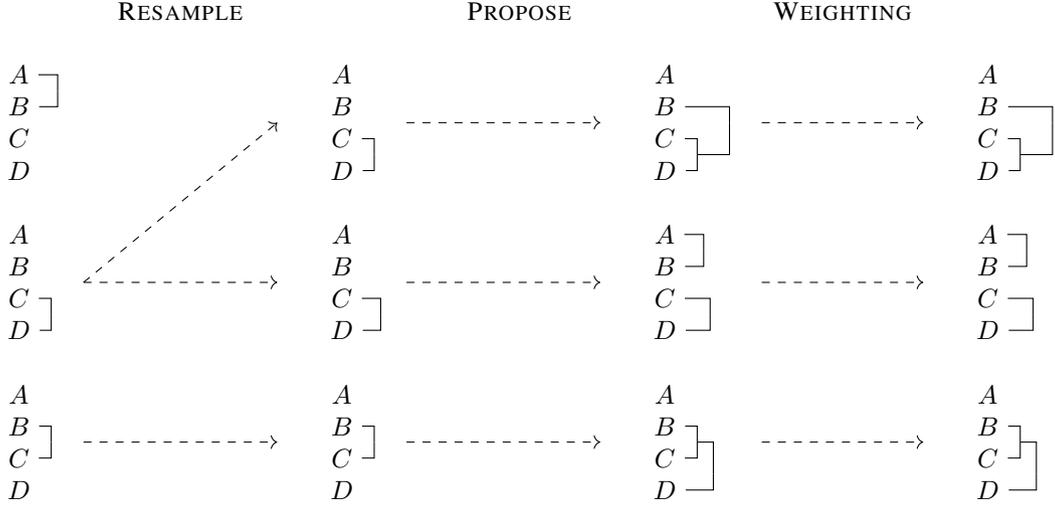
\begin{figure*}[ht!]
    \centering
    \begin{tikzpicture}[sloped, scale=0.85]
    \tikzset{
    dep/.style={circle,minimum size=1,fill=orange!20,draw=orange,
                general shadow={fill=gray!60,shadow xshift=1pt,shadow yshift=-1pt}},
    dep/.default=1cm,
    cli/.style={circle,minimum size=1,fill=white,draw,
                general shadow={fill=gray!60,shadow xshift=1pt,shadow yshift=-1pt}},
    cli/.default=1cm,
    obs/.style={circle,minimum size=1,fill=gray!20,draw,
                general shadow={fill=gray!60,shadow xshift=1pt,shadow yshift=-1pt}},
    obs/.default=1cm,
    spl/.style={cli=1,append after command={
                  node[circle,draw,dotted,
                       minimum size=1.5cm] at (\tikzlastnode.center) {}}},
    spl/.default=1cm,
    c1/.style={-stealth,very thick,red!80!black},
    v2/.style={-stealth,very thick,yellow!65!black},
    v4/.style={-stealth,very thick,purple!70!black}}
    
        \node(Resample) at (-4.5,6) {\textsc{Resample}};
        \node (A) at (-7,5) {$A$};
        \node (B) at (-7,4.5) {$B$};
        \node (C) at (-7,4) {$C$};
        \node (D) at (-7,3.5) {$D$};
        
        \node (ab) at ( -6.4,4.75) {};
        \draw (A) -| (ab.center);
        \draw (B) -| (ab.center);
        
        \coordinate (aone) at (-6,4.25); 
        
        \node (E) at (-7,2.5) {$A$};
        \node (F) at (-7,2) {$B$};
        \node (G) at (-7,1.5) {$C$};
        \node (H) at (-7,1) {$D$};
        \node (gh) at (-6.5,1.25) {};
        \draw (G) -| (gh.center) {};
        \draw (H) -| (gh.center) {};
        
        \coordinate (bone) at (-6, 1.75);
        
        \node (I) at (-7,0.) {$A$};
        \node (J) at (-7,-.5) {$B$};
        \node (K) at (-7,-1) {$C$};
        \node (L) at (-7,-1.5) {$D$};
        \node (jk) at (-6.5,-.75) {};
        \draw (J) -| (jk.center) {};
        \draw (K) -| (jk.center) {};
        
        \coordinate (cone) at (-6, -.75);
    
        \node(Propose) at (.75,6) {\textsc{Propose}};
        
        \node (AA) at (-2,5) {$A$};
        \node (BB) at (-2,4.5) {$B$};
        \node (CC) at (-2,4) {$C$};
        \node (DD) at (-2,3.5) {$D$};
        \node (ccdd) at (-1.5,3.75) {};
        \draw (CC) -| (ccdd.center);
        \draw (DD) -| (ccdd.center);
        
        \coordinate (atwo) at (-3, 4.25);
        \coordinate (athree) at (-1, 4.25);

        \node (EE) at (-2,2.5) {$A$};
        \node (FF) at (-2,2) {$B$};
        \node (GG) at (-2,1.5) {$C$};
        \node (HH) at (-2,1) {$D$};
        \node (gghh) at (-1.4,1.25) {};
        \draw (GG) -| (gghh.center);
        \draw (HH) -| (gghh.center);
        
        \coordinate (btwo) at (-3, 1.75);
        \coordinate (bthree) at (-1, 1.75);
        \draw[->,dashed] (bone) -- (btwo);
        \draw[->,dashed] (bone) -- (atwo);
        
        \node (II) at (-2,0) {$A$};
        \node (JJ) at (-2,-.5) {$B$};
        \node (KK) at (-2,-1) {$C$};
        \node (LL) at (-2,-1.5) {$D$};
        \node (kkll) at (-1.5, -.75) {};
        \draw (JJ) -| (kkll.center);
        \draw (KK) -| (kkll.center);
        
        \coordinate (ctwo) at (-3, -.75);
        \coordinate (cthree) at (-1, -.75);
        \draw[->,dashed] (cone) -- (ctwo);

        \node (U) at (3,5) {$A$};
        \node (V) at (3,4.5) {$B$};
        \node (W) at (3,4) {$C$};
        \node (X) at (3,3.5) {$D$};
        \node (wwxx) at (3.5, 3.75) {};
        \draw (W) -| (wwxx.center);
        \draw (X) -| (wwxx.center);
        \node (vwwxx) at (4, 4.25) {};
        \draw (wwxx.center) -| (vwwxx.center);
        \draw (V) -| (vwwxx.center);
        \coordinate (afour) at (2,4.25);
        \draw[->, dashed] (athree) -- (afour);
        
        \node (M) at (3,2.5) {$A$};
        \node (N) at (3,2) {$B$};
        \node (O) at (3,1.5) {$C$};
        \node (P) at (3,1) {$D$};
        \node (MN) at (3.6, 2.25) {};
        \draw (M) -| (MN.center);
        \draw (N) -| (MN.center);
        \node (OP) at (3.7, 1.25) {};
        \draw (O) -| (OP.center);
        \draw (P) -| (OP.center);
        \coordinate (bfour) at (2,1.75);
        \draw[->, dashed] (bthree) -- (bfour);

        \node (Q) at (3,0.) {$A$};
        \node (R) at (3,-.5) {$B$};
        \node (S) at (3,-1.) {$C$};
        \node (T) at (3,-1.5) {$D$};
        \node (RS) at (3.5, -.75) {};
        \draw (R) -| (RS.center);
        \draw (S) -| (RS.center);
        \node (RST) at (3.75,-1.25) {};
        \draw (RS.center) -| (RST.center);
        \draw (T) -| (RST.center);
        
        \coordinate (cfour) at (2,-.75);
        \draw[->, dashed] (cthree) -- (cfour);

        \node(Weighting) at (5.75,6) {\textsc{Weighting}};
        \node (Uu) at (8,5) {$A$};
        \node (Vv) at (8,4.5) {$B$};
        \node (Ww) at (8,4) {$C$};
        \node (Xx) at (8,3.5) {$D$};
        \node (wwxx) at (8.5, 3.75) {};
        \draw (Ww) -| (wwxx.center);
        \draw (Xx) -| (wwxx.center);
        \node (vwwxx) at (9., 4.25) {};
        \draw (Vv) -| (vwwxx.center);
        \draw (wwxx.center) -| (vwwxx.center);
        
        \coordinate (afive) at (4.5,4.25);
        \coordinate (asix) at (7,4.25);
         \draw[->, dashed] (afive) -- (asix);
        
        \node (Mm) at (8,2.5) {$A$};
        \node (Nn) at (8,2) {$B$};
        \node (Oo) at (8,1.5) {$C$};
        \node (Pp) at (8,1) {$D$};
        \node (MN) at (8.6, 2.25) {};
        \draw (Mm) -| (MN.center);
        \draw (Nn) -| (MN.center);
        \node (OP) at (8.7, 1.25) {};
        \draw (Oo) -| (OP.center);
        \draw (Pp) -| (OP.center);
        
        \coordinate (bfive) at (4.5,1.75);
        \coordinate (bsix) at (7,1.75);
        \draw[->, dashed] (bfive) -- (bsix);
        
        \node (Qq) at (8,0.) {$A$};
        \node (Rr) at (8,-.5) {$B$};
        \node (Ss) at (8,-1.) {$C$};
        \node (Tt) at (8,-1.5) {$D$};
        \node (RS) at (8.5, -.75) {};
        \draw (Rr) -| (RS.center);
        \draw (Ss) -| (RS.center);
        \node (RST) at (8.75, -1.25) {};
        \draw (RS.center) -| (RST.center);
        \draw (Tt) -| (RST.center);
        
        \coordinate (cfive) at (4.5,-.75);
        \coordinate (csix) at (7,-.75);
        \draw[->, dashed] (cfive) -- (csix);

    \end{tikzpicture}
    \caption{Overview of the \textsc{Csmc} framework. $K$ partial states are maintained as forests over the set of taxa. Each iteration of Algorithm \ref{alg:csmc} involves three steps: (1) resample partial states according to their importance weights, (2) propose an extension of each partial state to a new partial state by connecting two trees in the forest, and (3) compute weights for each new partial state by using Felsenstein's pruning algorithm. In the above, three samples are shown over four taxa $A,B,C,D.$}
    \label{fig:trellis}
\end{figure*}

\paragraph{Phylogenetic Trees.}
We wish to infer a latent bifurcating tree that describes the evolutionary relationships among a set of observed molecular sequences. A phylogeny is defined by a tree topology $\tau$ and a set of branch lengths $\mathcal{B}$.
A \textit{tree topology} is defined as a connected acyclic graph $(V,E)$ where $V$ is a set of vertices and $E$ is a set of edges. \textit{Leaf nodes} denote vertices of degree 1 and correspond to observed taxa. \textit{Internal nodes} designate vertices of degree 3 (one parent and two children) and represent unobserved taxa (e.g. DNA bases of ancestral species). The \textit{root node} is of degree 2 (two children) and represents the common evolutionary ancestor of all taxa. 

For each edge $e \in E$, we associate a \textit{branch length}, denoted $b(e) \in \mathbb{R}_{>0}$, and $\mathcal{B} = \{b(e)\}_{e \in E}$. The branch length captures the intensity of the evolutionary changes between two vertices. 
An \textit{ultrametric tree} is one with constant evolutionary rate along all paths from $v$ to its descendants. 
\textit{Nonclock trees} are general trees that do not require ultrametric assumptions. In this work we focus on phylogenetic inference methods for nonclock trees as these are most pertinent to biologists.

\paragraph{Bayesian Phylogenetic Inference.}
Let the matrix $\mathbf{Y} = \{Y_1,\cdots,Y_S \} \in \Omega^{NxS}$ denote the observed molecular sequences with characters in $\Omega$ of length $S$ over $N$ species. Bayesian inference requires specifying the prior density and likelihood function over tree topology $\tau$, branch length set $\mathcal{B}$ and generative model parameters $\theta$ to write the joint posterior,
\begin{equation}
    P_{\theta}(\mathcal{B},\tau|\mathbf{Y}) = \frac{P_{\theta}(\mathbf{Y}|\tau,\mathcal{B})P_{\theta}(\tau,\mathcal{B})}{P_{\theta}(\mathbf{Y})}.
    \label{eq:posterior}
\end{equation}
The prior is uniform over topologies and a product of independent exponential distributions over branch lengths with rate $\lambda_{bl}$. The evolution of each site is modeled independently using a continuous time Markov chain with rate matrix $\mathbf{Q}$.
Let $\zeta_{v,s}$ denote the state of genome for species $v$ at site $s$ and define the evolutionary model along branch $b(v\rightarrow v')$:
\begin{equation} 
P_{\theta}(\zeta_{v',s} = j| \zeta_{v,s} = i) = \exp\left(b(e)\mathbf{Q}_{i,j}\right).
\end{equation}
The likelihood of a given phylogeny $P_{\theta}(\mathbf{Y}|\tau,\mathcal{B}) = \prod\limits_{i=1}^{S}P_{\theta}(Y_i|\tau,\mathcal{B})$ can be evaluated in linear time using the sum-product or Felsenstein's pruning algorithm~\citep{Felsenstein:1981:J-Mol-Evol:7288891} via the formula:
\begin{align*}
    P_{\theta}(&\mathbf{Y}|\tau,\mathcal{B}) 
    \coloneqq \prod\limits_{i=1}^{S}\sum\limits_{a^i}^{}\eta(a^i_{\rho})\prod\limits_{(u,v)\in E(\tau)}^{}\text{exp}\left(-b_{u,v}\mathbf{Q}_{a_u^i,a_v^i} \right),
\end{align*}
where $\rho$ is the root node, $a^i_u$ is the assigned character of node $u$, $E(\tau)$ represents the set of edges in $\tau$ and $\eta$ is the prior or stationary distribution of the Markov chain. The normalization constant $P_{\theta}(\mathbf{Y})$ requires marginalizing the $(2N-3)!!$ distinct topologies which is intractable~\citep{semple2003phylogenetics}.

\paragraph{Computational Challenges.} We distinguish the two computational tasks required for phylogenetic inference. First, inference involves computing the normalization constant $P_{\theta}(\mathbf{Y})$ by marginalizing the $(2N-3)!!$ distinct topologies:
    \begin{equation}
        P_\theta(\mathbf{Y}) = \sum\limits_{\tau \in \mathcal{T}}^{}\int p_\theta(\mathbf{Y}|\tau,\mathcal{B})p_\theta(\tau,\mathcal{B}) d\mathcal{B} \,.
        \label{eq:likelihood}
    \end{equation}
A common approach used for approximating Eq. \ref{eq:likelihood} is to sample tree topologies $\tau$ and branch lengths $\mathcal{B}$ via Monte Carlo methods, such as \textsc{Csmc}, given that $\theta$ is known.

Second, learning (\textit{optimization}) refers to finding the set of parameters $\theta = (Q,\{\lambda_i\}_{i=1}^{|E|} \in \Theta)$ that maximize the data log-likelihood obtained by marginalizing Eq. \ref{eq:likelihood}:
    \begin{equation}
        \theta^\star =  \argmax_{Q,\{\lambda_i\}_{i=1}^{|E|}} ~{\log P_\theta(\mathbf{Y})} \,. 
    \end{equation}
Sampling algorithms can also be used by assigning a prior to $\theta$, then performing a local search for the parameters via \textsc{Mcmc} methods, given that the data likelihood is available.

\paragraph{Variational Inference.}
\textsc{Vi} is a technique for approximating the posterior $P_\theta(\mathcal{B},\tau | \mathbf{Y})$ 
when marginalization of latent variables is not analytically feasible. By introducing a tractable distribution $Q_\phi(\mathcal{B},\tau|\mathbf{Y})$ it is possible to form a lower bound to the log-likelihood:
\begin{equation}
\log P_\theta(\mathbf{Y}) \geq \mathcal{L}_{\text{ELBO}}(\theta,\phi,\mathbf{Y}) \coloneqq \underset{Q}{\mathbb{E}}\Bigg[\log \frac{P_\theta(\mathbf{Y}, \mathcal{B},\tau)} {Q_\phi(\mathcal{B},\tau|\mathbf{Y})}\Bigg] \,. \label{ELBO}
\end{equation}
Auto Encoding Variational Bayes~\citep{kingma2013autoencoding} (\textsc{Aevb}) simultaneously 
trains
$Q_\phi(\mathcal{B},\tau|\mathbf{Y})$ and $P_\theta(\mathbf{Y}, \mathcal{B},\tau)$. 
The expectation in Eq. \ref{ELBO} is approximated by averaging Monte Carlo samples from $Q_\phi(\mathcal{B},\tau|\mathbf{Y})$ which are reparameterized by evaluating a deterministic function of a $\phi$-independent random variable. When the ratio $P_\theta(\mathbf{Y}, \mathcal{B},\tau)/Q_\phi(\mathcal{B},\tau|\mathbf{Y})$ is concentrated around its mean, Jensen's inequality produces a tighter bound.

Deriving a tractable approximation $Q_\phi(\mathcal{B},\tau|\mathbf{Y})$ for the phylogenetic tree model can be challenging so we turn to \textsc{Csmc}.

\paragraph{Combinatorial Sequential Monte Carlo.}
 \textsc{Csmc} is designed for inference in phylogenetic tree models. \textsc{Csmc} approximates a sequence of target distributions $\bar{\pi}_r$ on increasing probability spaces such that the final target coincides with Eq. \ref{eq:posterior}~\citep{csmc}. The (unnormalized) target distribution  $\pi$ and its normalization constant $\|\pi\|$ corresponding to the numerator and denominator in Eq. \ref{eq:posterior}  are approximated by sequential importance resampling in $R$ steps. Unlike standard \textsc{Smc} methods, the target $\pi$ is defined on a combinatorial set (the space of tree topologies) and the continuous branch lengths. This requires defining an intermediate object referred to as a \textit{partial state}.%
  \begin{definition}[Partial State]
 A partial state of rank r denoted $s_r=\{(t_i,X_i)\}$ is a collection of rooted trees that satisfies the following three conditions: (i) the set of partial states of different ranks are disjoint, $\forall\, r \neq s$, $\mathcal{S}_r \cap \mathcal{S}_s = \emptyset$ ; (ii) the set of partial states of smallest rank has a single element $S_0 = \{\bot \}$; and (ii) the set of partial states at the final rank $R$ corresponds to the target space $\mathcal{X}$.
 \end{definition}
 \textsc{Csmc} operates by sampling $K$ \textit{partial states} (or \textit{particles}) $\{s_{r}^k\}_{k=1}^{K} \in \mathcal{S}_r$ at each rank $r$ which are used to form a distribution,
\begin{equation}
    \widehat{\pi}_{r} = \|\widehat{\pi}_{r-1}\|\frac{1}{K}\sum\limits_{k=1}^{K}w_{r}^k\delta_{s_r^k}(s) \qquad \forall s \in \mathcal{S},
\end{equation}
 where $\delta_s$ is the Dirac measure and $w_{r}^k$ are the importance weights. 
 Resampling ensures that particles remain in areas of high probability mass. Each resampled state $s_{r-1}^{a_{r-1}^k}$, where $a_{r-1}^k \in \{1,\ldots,K\}$ is the resampled index, of rank $r-1$ is then extended to a state of rank $r$, $s_{r}^k$, by simulating from a proposal distribution $q(\cdot |s_{r-1}^{a_{r-1}^k}): \mathcal{S} \rightarrow [0,1]$. The importance weights are computed as follows:
 \begin{equation}
 w_{r}^k = w(s_{r-1}^{a_{r-1}^k},s_{r}^k) = \frac{\pi(s_{r}^k)}{\pi(s_{r-1}^{a_{r-1}^k})}\cdot \frac{\nu^{-}(s_{r-1}^{a_{r-1}^k})}{q(s_{r}^k |s_{r-1}^{a_{r-1}^k})},
 \end{equation}
 where $\nu^{-}$ is a probability density over $\mathcal{S}$ correcting an over-counting problem~\citep{csmc}. An overview of the procedure is given in Fig. \ref{fig:trellis}. 
 An unbiased estimate for the marginal likelihood can be constructed from the weights which converges in $L^2$ norm,
\begin{equation}
    \widehat{\mathcal{Z}}_{CSMC} \coloneqq \|\widehat{\pi}_{R}\| = \prod\limits_{r=1}^{R}\left(\frac{1}{K} \sum\limits_{k=1}^{K}w_{r}^k\right) \rightarrow \|\pi \|.
    \label{eq:smcmarginallikelihood}
\end{equation}

\textsc{Vcsmc} melds \textsc{Vi} and \textsc{Csmc} to approximate the posterior as well as the model parameters.

\section{Variational Combinatorial Sequential Monte Carlo}
\label{vcsmc}
\paragraph{Variational Objective.} 
The idea of \textsc{Vcsmc} is to simultaneously learn the model parameters and proposal parameters by maximizing a lower bound to the data marginal log-likelihood, using \textsc{Csmc} 
as an unbiased estimator of the marginal likelihood. 

We begin by defining a structured approximate posterior which factorizes over rank events. Each state (or rank event) $s_r$ is specified by a topology, a forest of trees, and their corresponding set of branch lengths. The proposal $q_{\phi,\psi}(s_{r}^k|s_{r-1}^{a_{r-1}^k})$ is the probability of state $s_{r}^k$ given the resampled state at the previous rank $s_{r-1}^{a_{r-1}^k}$. Subscripts  $\phi$ and $\psi$ denote discrete and continuous proposal parameters respectively. The approximate posterior  is (written explicitly in Eq.~\ref{eq:fullposterior}):
\begin{align}
     &Q_{\phi,\psi}\left(s_{1:R}^{1:K}, a_{1:R-1}^{1:K}\right)
    \coloneqq   \\ 
    &
    \prod\limits_{k=1}^{K}q_{\phi,\psi}(s_{1}^k)\times 
    \prod\limits_{r=2}^{R}\prod\limits_{k=1}^{K} \left[
    \frac{w_{r-1}^{a_{r-1}^k}}{\sum_{l=1}^K w_{r-1}^l} \cdot
    q_{\phi,\psi}\left(s_{r}^k|s_{r-1}^{a_{r-1}^k}\right)
     \right]
    \, . \nonumber
\end{align}

At the final rank event $R=N-1$, an unbiased approximation to the likelihood is formed by averaging over importance weights, which, in turn represent the sample phylogenies that are constructed iteratively. A multi-sample variational objective is formed via the lower bound:
\begin{align}
     \mathcal{L}_{CSMC} 
     \coloneqq \underset{Q}{\mathbb{E}}\left[\log \widehat{ \mathcal{Z}}_{CSMC} \right]\,.
\end{align}
The presence of the discrete distribution over partial states presents a challenge for variational reparameterization. Unlike standard variational SMC methods \citep{pmlr-v84-naesseth18a}, states are formed by sampling from a large combinatorial set. We take two approaches, the first is to drop discrete terms from the gradient estimates. The second is to reparameterize these terms as Gumbel-Softmax random variables forming a differentiable approximation through a convex relaxation over the simplex. Continuous proposal terms are drawn by evaluating a deterministic function of a $\psi$-independent random variable.

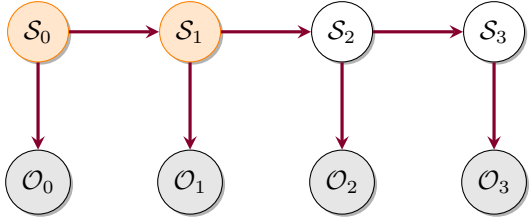
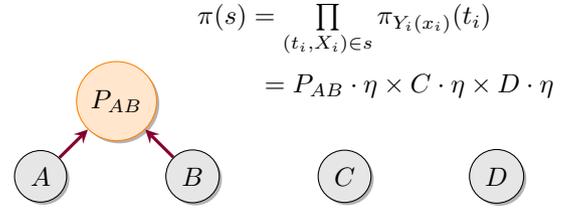
\begin{figure*}
\subfigure[State space model representation of \textsc{Csmc}. Latent variables consist of internal nodes and branch lengths with Markovian dependencies whereas observations are the recorded molecular sequences.]{
    \begin{tikzpicture}
    \tikzset{
    dep/.style={circle,minimum size=1,fill=orange!20,draw=orange,
                general shadow={fill=gray!60,shadow xshift=1pt,shadow yshift=-1pt}},
    dep/.default=1cm,
    cli/.style={circle,minimum size=1,fill=white,draw,
                general shadow={fill=gray!60,shadow xshift=1pt,shadow yshift=-1pt}},
    cli/.default=1cm,
    obs/.style={circle,minimum size=1,fill=gray!20,draw,
                general shadow={fill=gray!60,shadow xshift=1pt,shadow yshift=-1pt}},
    obs/.default=1cm,
    spl/.style={cli1,append after command={
                  node[circle,draw,dotted,
                       minimum size=1.5cm] at (\tikzlastnode.center) {}}},
    spl/.default=1cm,
    c1/.style={-stealth,very thick,red!80!black},
    v2/.style={-stealth,very thick,yellow!65!black},
    v4/.style={-stealth,very thick,purple!70!black}}
\begin{scope}[local bounding box=graph]
    \node[dep] (1) at (2,2) {$\mathcal{S}_0$};
    \node[dep] (2) at (4,2) {$\mathcal{S}_1$};
    \node[cli] (3) at (6,2) {$\mathcal{S}_2$};
    \node[cli] (4) at (8,2) {$\mathcal{S}_3$};
    \node[obs] (1b) at (2,0) {$\mathcal{O}_0$};
    \node[obs] (2b) at (4,0) {$\mathcal{O}_1$};
    \node[obs] (3b) at (6,0) {$\mathcal{O}_2$};
    \node[obs] (4b) at (8,0) {$\mathcal{O}_3$};
    \draw[v4] (1) -- (2);
    \draw[v4] (2) -- (3);
    \draw[v4] (3) -- (4);
    \draw[v4] (1) -- (1b);
    \draw[v4] (2) -- (2b);
    \draw[v4] (3) -- (3b);
    \draw[v4] (4) -- (4b);
\end{scope}
\end{tikzpicture}
}
\hfill
\subfigure[Overview of the natural forest extension of the target measure. Partial state $s_{r}^k=\{P_{AB},C,D\}$ is defined as a forest over leaves $\{A,B,C,D\}$. ]{
    \begin{tikzpicture}
    \tikzset{
    dep/.style={circle,minimum size=1,fill=orange!20,draw=orange,
                general shadow={fill=gray!60,shadow xshift=1pt,shadow yshift=-1pt}},
    dep/.default=1cm,
    cli/.style={circle,minimum size=1,fill=white,draw,
                general shadow={fill=gray!60,shadow xshift=1pt,shadow yshift=-1pt}},
    cli/.default=1cm,
    obs/.style={circle,minimum size=1.,fill=gray!20,draw,
                general shadow={fill=gray!60,shadow xshift=1pt,shadow yshift=-1pt}},
    obs/.default=1cm,
    spl/.style={cli=1,append after command={
                  node[circle,draw,dotted,
                       minimum size=1.5cm] at (\tikzlastnode.center) {}}},
    spl/.default=1cm,
    c1/.style={-stealth,very thick,red!80!black},
    v2/.style={-stealth,very thick,yellow!65!black},
    v4/.style={-stealth,very thick,purple!70!black}}
\begin{scope}[local bounding box=graph]
    \node[dep] (Pab) at (-1,-2) {$P_{AB}$};
    \node[obs] (A) at (-2,-3) {$A$};
    \node[obs] (B) at (0,-3) {$B$};
    \node[obs] (C) at (2,-3) {$C$};
    \node[obs] (D) at (4,-3) {$D$};
    \draw[v4] (A) -- (Pab);
    \draw[v4] (B) -- (Pab);
    \node[draw=none] at (2,-1) {$\pi(s) = \prod\limits_{(t_i, X_i)\in s}^{} \pi_{Y_i(x_i)}(t_i) $};
    \node[draw=none] at (2.85, -1.8) {$ = P_{AB}\cdot \eta \times C \cdot \eta \times D \cdot \eta$};
\end{scope}
\end{tikzpicture}
}
\caption{Illustration of the sequence of probability spaces along with the natural forest extension used by $\textsc{Csmc}$. The probability measure $\pi$ is defined on the target space of trees $\mathcal{S}_R$, and not the larger sample space of partial states $ \mathcal{S}_{r<R}$, which are defined on forests. The partial state $s_{r}^{k} = \{P_{AB},C,D\}$ corresponding to $\mathcal{S}_1$ of Fig. 2 (a) is illustrated in Fig. 2 (b) as a set of disjoint components over the four taxa $\{A,B,C,D\}$. Felsenstein's pruning algorithm is used to obtain a marginal likelihood estimate for each tree by passing messages from left and right child nodes and taking the inner product with $\eta$, the stationary state of $\mathbf{Q}$. Each distinct likelihood is then multiplied to assign probability $\pi(s_{r}^k)$ to partial state $s_{r}^k$.}
\label{fig:nfe}
\end{figure*}

\paragraph{Implementation Details.}
Constructing the objective $\mathcal{L}_{CSMC}$ is done iteratively in three steps. The proposal procedure, $q(s_r|s_{r-1})$, requires selecting two trees to coalesce by sampling without replacement. This is accomplished by defining Gumbel-Softmax random variables. The uniform log-probability for each index is perturbed by adding independent Gumbel distributed noise, after which the largest two elements are returned. For example let $U\sim \textsc{Uniform}(0,1)$, we then form $G = \gamma - \log(- \log U)$ so that $G$ can be reparameterized as $G' = G + \gamma$. The \textsc{Resample} procedure 
can also be reparameterized similarly by defining Gumbel-Softmax random variables. 

\paragraph{Extending the Target Measure.} The \textsc{weighting} step requires some care. In order to compute importance weights, the likelihood of a partial state must be evaluated using Felsenstein's pruning algorithm, however the likelihood of Eq. \ref{eq:likelihood} and the probability measure $\pi$ are defined on the target space of trees $\mathcal{S}_R$, and not the larger sample space of partial states $ \mathcal{S}_{r<R}$, which are defined on forests (trees disjoint from each other). The pruning algorithm yields a maximum likelihood estimate for an evolutionary tree, but partial states are defined as collections of disjoint trees or leaf nodes. 
One extension of the target measure $\pi$ into a measure on $\mathcal{S}_{r<R}$ is to treat all elements of the jump chain as trees~\citep{csmc}. The contribution of each of the trees to the likelihood is multiplied 
by taking the inner product of each distribution over characters with $\eta$. 
\begin{definition}[Natural Forest Extension]
 The natural forest extends target measure $\pi$ into forests by taking a product over the trees in the forest:
\begin{equation}
    \pi(s) \coloneqq \prod\limits_{(t_i,X_i)}^{}\pi_{Y_i(x_i)}(t_i)\,.
\end{equation}
\end{definition}
The natural forest extension (\textsc{Nfe}) has the advantage of passing information from the non-coalescing elements to the local weight update. Fig. \ref{fig:nfe} provides an illustration of the \textsc{Nfe} applied to the state consisting of  $\textsc{Pa}(A,B)$ and non-coalescing singletons $\{C\}$ and $\{D\}$.

\section{Nested Combinatorial Sequential Monte Carlo}

A potential drawback of the \textsc{Csmc} method is that partial states are sampled to coalesce uniformly, when many of the resulting topologies correspond to areas of low probability mass. It seems natural to incorporate information from future iterations within the proposal distribution to subsequently guide the exploration of partial states. Adapting the proposal requires marginalizing the intermediate target over future topologies and branch lengths.


\paragraph{Locally Optimal Combinatorial \textsc{Smc}.} 
Choosing a good proposal distribution is key for the effectiveness of \textsc{Smc} methods. The \emph{locally optimal} \textsc{Smc} \citep{doucet2000sequential,naesseth2019elements} chooses the proposal in such a way that all particles have equal weights. This can significantly improve the performance over the standard proposal used in \textsc{Csmc}. The locally optimal proposal based on the natural forest extension is
\begin{align}
    q(s_r  | s_{r-1}) &\propto \frac{\pi(s_r) \nu^-(s_{r-1})}{\pi(s_{r-1})}.
\end{align}
This locally optimal proposal for the \textsc{Csmc} algorithm is computationally intractable, it requires us to exactly marginalize the branch lengths. We use the nested \textsc{Smc} \citep{naesseth2015nested,naesseth2016high} method to overcome this problem.

\paragraph{Nested Combinatorial \textsc{Smc}.}
We provide an overview of Nested Combinatorial Sequential Monte Carlo before presenting a detailed description in Algorithm~\ref{alg:ncsmc}  (we have annotated the overview with steps from the algorithm). \textsc{Ncsmc} iterates over rank events (\textit{line 2}) to perform a standard \textsc{Resample} step also used in \textsc{Csmc} methods (\textit{line 4}). For each sample, \textsc{Ncsmc} enumerates all ${N -r \choose 2}$ possible one-step ahead topologies and samples corresponding  $M$ \textit{sub-branch} lengths (\textit{line 7}). 
We evaluate importance \textit{sub-weights} or \textit{potential functions} for each of these $s_{r}^{k,m}[i]$ sampled look-ahead states (\textit{line 8}). Then, we extend our ancestral partial state $s_{r-1}^{a_{r-1}^k}$ to the new partial state $s_r^k$ (\textit{line 11}) by selecting one of the topologies and a corresponding branch length according to its weight.
Finally, for each sample (\textit{line 12}), we compute its weight by averaging over all the potential functions. An illustration of the procedure is given in Fig.~\ref{fig:ncsmc} of the Appendix.

\paragraph{Variational Nested \textsc{Csmc} Objective.} 
The nested \textsc{Csmc} method described in Algorithm~\ref{alg:ncsmc} can also be used to construct a variational objective:
\begin{align}
    \mathcal{L}_{NCSMC} &
     \coloneqq \underset{Q}{\mathbb{E}}\left[\log \hat{ \mathcal{Z}}_{NCSMC} \right]\,,\\ 
      \widehat{\mathcal{Z}}_{NCSMC} &
      \coloneqq 
      \prod\limits_{r=1}^{R}\left(\frac{1}{K} \sum\limits_{k=1}^{K}w_{r}^k\right).
\end{align}
We refer to the resulting \textsc{Vi} framework as \textsc{Vncsmc}.

\paragraph{Theoretical Justification.}
Nested \textsc{Csmc} is an \textsc{Smc} algorithm on the extended space of all random variables generated by Algorithm~\ref{alg:ncsmc}. This means it keeps keeps the favorable properties of \textsc{Csmc}, such as unbiasedness of the normalization constant estimate and asymptotic consistency. The key property that ensures this for \textsc{Ncsmc} is \emph{proper weighting} \citep{naesseth2015nested,naesseth2016high}. 
\begin{definition}[Proper Weighting]
We say that the random pair $(s_r, w_r)$ are \emph{properly weighted} for the unnormalized distribution $\frac{\pi(s_r) \nu^-(s_{r-1})}{\pi(s_{r-1})}$ if $w_r \geq 0$ almost surely, and for all measurable functions $h$,
\begin{align}
    \mathbb{E}[w_r h(s_r)] &= \int h(s_r) \frac{\pi(s_r) \nu^-(s_{r-1})}{\pi(s_{r-1})} \, \mathrm{d}s_r.
\end{align}
\end{definition}
We formalize the result for \textsc{Ncsmc}, Algorithm~\ref{alg:ncsmc}, in Theorem~\ref{thm:pw}. We say that nested \textsc{Csmc} is an \emph{exact approximation} \citep{naesseth2019elements} of \textsc{Csmc} with the locally optimal proposal.
\begin{theorem}
The particles $s_r^k$ and weights $w_r^k$ generated by Algorithm~\ref{alg:ncsmc} are properly weighted for $\frac{\pi(s_r) \nu^-(s_{r-1})}{\pi(s_{r-1})}$.
\label{thm:pw}
\end{theorem}
\begin{proof}
\begin{align}
    &\mathbb{E}[w_r^k h(s_r^k)] = \mathbb{E}\left[w_r^k \cdot h(s_r^{k,J}[I])\right]   \nonumber \\
    &= \mathbb{E}\left[\sum_{i=1}^L \sum_{j=1}^M w_r^k \frac{w_r^{k,j}[i]}{\sum_l \sum_m w_r^{k,m}[l]} h(s_r^{k,j}[i])\right] \nonumber \\
    &= \frac{1}{ML} \sum_{i=1}^L \sum_{j=1}^M\mathbb{E}\left[w_r^{k,j}[i] \cdot h(s_r^{k,j}[i])\right] \nonumber \\
    &= \mathbb{E}\left[w_r^{k,j}[i] \cdot h(s_r^{k,j}[i])\right] = \int h(s_r) \frac{\pi(s_r) \nu^-(s_{r-1})}{\pi(s_{r-1})} \, \mathrm{d}s_r \nonumber
\end{align}
\end{proof}

\begin{algorithm}[tb]
   \caption{Nested Combinatorial Sequential Monte Carlo}
   \label{alg:ncsmc}
    \begin{algorithmic}
   \STATE {\bfseries Input:} {$\mathbf{Y} = \{Y_1,\cdots,Y_M \} \in \Omega^{NxM}$,  $\theta = (\mathbf{Q},\{\lambda_i\}_{i=1}^{|E|})$}
   \end{algorithmic}
   \begin{algorithmic}[1]
   \STATE Initialization. $\forall k$, $s_{0}^k\leftarrow \perp$, $w_{0}^k\leftarrow 1/K$. 
   \FOR{$r=1$ {\bfseries to} $R=N-1$}
   \FOR{$k=1$ {\bfseries to} $K$}
    \STATE \textsc{Resample}~~ 
    $\mathbb{P}(a_{r-1}^k = i) = \frac{w_{r-1}^i}{\sum_{l=1}^K w_{r-1}^l}$
   \FOR{$i=1$ {\bfseries to} $L = {N-r \choose 2}$}
   \FOR{$m=1$ {\bfseries to} $M$}
   \STATE \textsc{Form look-ahead partial state}
   \[ s_{r}^{k,m}[i] \sim q(\cdot|s_{r-1}^{a_{r-1}^k}) \] 
   
   \STATE \textsc{Compute potentials} \[ w_{r}^{k,m}[i] =  \frac{\pi(s_{r}^{k,m}[i])}{\pi(s_{r-1}^{a_{r-1}^k})}\cdot \frac{\nu^{-}(s_{r-1}^{a_{r-1}^k})}{q(s_{r}^{k,m}[i]|s_{r-1}^{a_{r-1}^k})} 
   \]
   \ENDFOR
   \ENDFOR
   \STATE \textsc{Extend partial state}
   \begin{align*}
      s_r^k &= s_r^{k,J}[I], \\
      \mathbb{P}(I = i, J=j) &= \frac{w_r^{k,j}[i]}{\sum_{l=1}^L \sum_{m=1}^M w_r^{k,m}[i]}
   \end{align*}
    \STATE \textsc{Compute weights} \[
    w_r^k = \frac{1}{ML}\sum_{i=1}^L \sum_{m=1}^M w_r^{k,m}[i] 
    \]
   \ENDFOR
   \ENDFOR
    \end{algorithmic}
    \begin{algorithmic}
   \STATE {\bfseries Output:} $s_{R}^{1:K}$ , $w_{1:R}^{1:K}$
   \end{algorithmic}
    
\end{algorithm}

\begin{figure*}[!ht]
\centering
\includegraphics[width=1.0\textwidth]{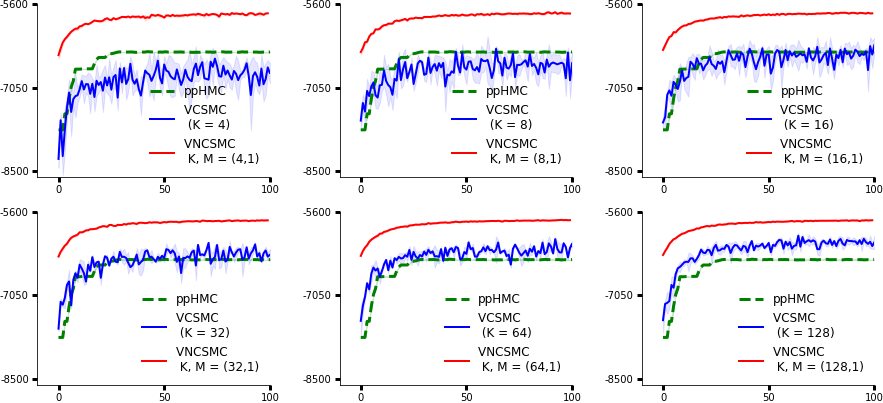}
\caption{Log likelihood values for $\textsc{Vcsmc}$ (blue) with $K$ = $\{4,8,16,32,64,128\}$ samples and $\textsc{Vncsmc}$ (red) with $K = \{4,8,16,32,64,128\}$ and $M=1$ samples on the primates data averaged across 5 random seeds. Higher values of $K$ produce tighter ELBO / larger log likelihood values with lower stochastic gadient noise. $\textsc{Vcsmc}$ with $K \geq 16$ outperforms probabilistic path Hamiltonain Monte Carlo ($ppHMC$) which is shown (green trace) for comparison. $\textsc{Vncsmc}$ requires fewer epochs than $\textsc{Vcsmc}$ to converge and produces tighter ELBO / larger log likelihood values with lower stochastic gadient noise. $\textsc{Vncsmc}$ with $(K,M)=(4,1)$ (top left) outperforms both $ppHMC$ and $\textsc{Vcsmc}$ with $ K = 128$ (bottom right).}
\label{fig:primatesloglik}
\end{figure*} 
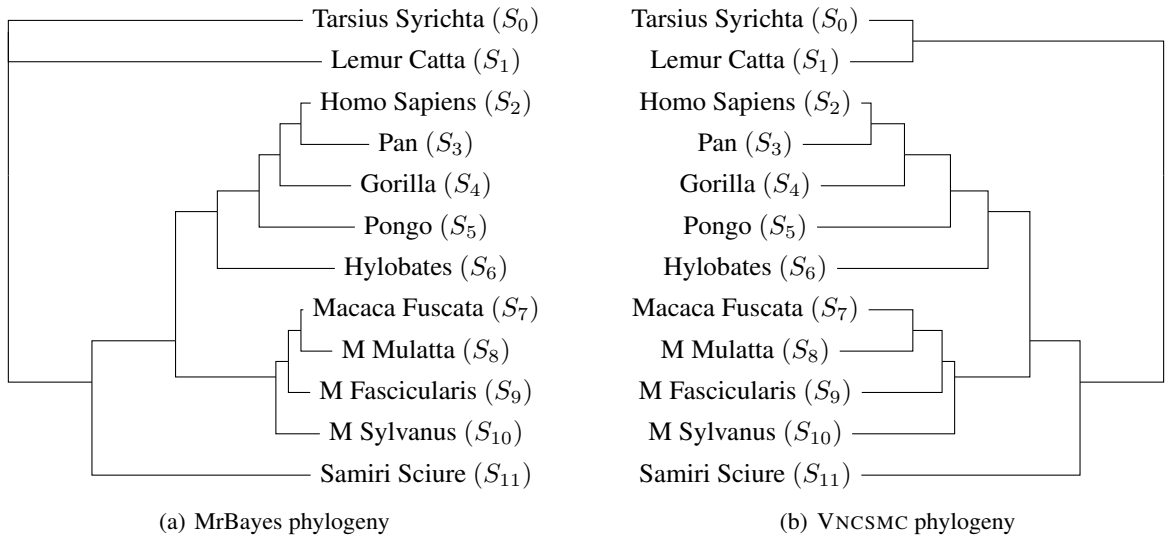
\begin{figure*}[!ht]  
\centering  

\subfigure[MrBayes phylogeny]  
{  
\begin{tikzpicture}[scale=0.55, rotate=90]

\node (a) at (-5.5,0) {Samiri Sciure $(S_{11})$};
\node (b) at (-4.5,0) {M Sylvanus $(S_{10})$};
\node (c) at (-3.5,0) {M Fascicularis $(S_{9})$};
\node (d) at (-2.5,0) {M Mulatta $(S_{8})$};
\node (e) at (-1.5,0) {Macaca Fuscata $(S_{7})$};

\node (de) at (-2,3) {};
\draw (d) |- (de.center);
\draw (e) |- (de.center);
\node (cde) at (-2.75, 3.3) {};
\draw (de.center) |- (cde.center);
\draw (c) |- (cde.center);
\node (bcde) at (-3.125, 3.6) {};
\draw (cde.center) |- (bcde.center);
\draw (b) |- (bcde.center);

\node (f) at (-.5,0) {Hylobates $(S_{6})$};
\node (g) at (.5,0) {Pongo $(S_{5})$};
\node (h) at (1.5,0) {Gorilla $(S_{4})$};
\node (i) at (2.5,0) {Pan $(S_{3})$};
\node (j) at (3.5,0) {Homo Sapiens $(S_{2})$};

\node (ij) at (3, 3) {};
\draw (i) |- (ij.center);
\draw (j) |- (ij.center);
\node (ijh) at (2.25, 3.5) {};
\draw (ij.center) |- (ijh.center);
\draw (h) |- (ijh.center);
\node (ijhg) at (1.375, 4) {};
\draw (g) |- (ijhg.center);
\draw (ijh.center) |- (ijhg.center);
\node (ijhgf) at (.875, 5) {};
\draw (f) |- (ijhgf.center);  
\draw (ijhg.center) |- (ijhgf.center);  

\node (bcdeijhgf) at (-2.25, 6) {};
\draw (ijhgf.center) |- (bcdeijhgf.center);
\draw (bcde.center) |- (bcdeijhgf.center);

\node (abcdeijhgf) at (-3.25, 8) {};
\draw (bcdeijhgf.center) |- (abcdeijhgf.center); 
\draw (a) |- (abcdeijhgf.center); 

\node (k) at (4.5,0) {Lemur Catta $(S_{1})$};
\node (l) at (5.5,0) {Tarsius Syrichta $(S_0)$};
\node (kl) at (5, 10) {};
\draw (k) |- (kl.center);
\draw (l) |- (kl.center);

\node (abcdeijhgfkl) at (1.75, 10) {};
\draw (abcdeijhgf.center) |- (abcdeijhgfkl.center);
\draw (kl) |- (abcdeijhgfkl.center);

\end{tikzpicture}
}  
\qquad
\subfigure[\textsc{Vncsmc} phylogeny]  
{  

\begin{tikzpicture}[scale=.55, rotate=90]

\node (a) at (-5.5,10) {Samiri Sciure $(S_{11})$};
\node (b) at (-4.5,10) {M Sylvanus $(S_{10})$};
\node (c) at (-3.5,10) {M Fascicularis $(S_{9})$};
\node (d) at (-2.5,10) {M Mulatta $(S_{8})$};
\node (e) at (-1.5,10) {Macaca Fuscata $(S_{7})$};

\node (de) at (-2,6) {};
\draw (d) |- (de.center);
\draw (e) |- (de.center);
\node (cde) at (-2.75, 5.3) {};
\draw (de.center) |- (cde.center);
\draw (c) |- (cde.center);
\node (bcde) at (-3.125, 5) {};
\draw (cde.center) |- (bcde.center);
\draw (b) |- (bcde.center);

\node (f) at (-.5,10) {Hylobates $(S_{6})$};
\node (g) at (.5,10) {Pongo $(S_{5})$};
\node (h) at (1.5,10) {Gorilla $(S_{4})$};
\node (i) at (2.5,10) {Pan $(S_{3})$};
\node (j) at (3.5,10) {Homo Sapiens $(S_{2})$};

\node (ij) at (3, 7) {};
\draw (i) |- (ij.center);
\draw (j) |- (ij.center);
\node (ijh) at (2.25, 6.2) {};
\draw (ij.center) |- (ijh.center);
\draw (h) |- (ijh.center);
\node (ijhg) at (1.375, 5.1) {};
\draw (g) |- (ijhg.center);
\draw (ijh.center) |- (ijhg.center);
\node (ijhgf) at (.875, 4.2) {};
\draw (f) |- (ijhgf.center);  
\draw (ijhg.center) |- (ijhgf.center);  

\node (bcdeijhgf) at (-2.25, 3.2) {};
\draw (ijhgf.center) |- (bcdeijhgf.center);
\draw (bcde.center) |- (bcdeijhgf.center);

\node (abcdeijhgf) at (-3.25, 2) {};
\draw (bcdeijhgf.center) |- (abcdeijhgf.center); 
\draw (a) |- (abcdeijhgf.center); 

\node (k) at (4.5,10) {Lemur Catta $(S_{1})$};
\node (l) at (5.5,10) {Tarsius Syrichta $(S_0)$};
\node (kl) at (5, 6) {};
\draw (k) |- (kl.center);
\draw (l) |- (kl.center);

\node (abcdeijhgfkl) at (1.75, 0) {};
\draw (abcdeijhgf.center) |- (abcdeijhgfkl.center);
\draw (kl.center) |- (abcdeijhgfkl.center);

\end{tikzpicture}

}
\caption{ MrBayes vs \textsc{Vncsmc} phylogeny on the primate mitochondrial DNA dataset. The data consists of 12 taxa $\{S_0,\cdots, S_{11}\}$ over 898 sites on the genome. 
The maximum likelihood topology returned by \textsc{Vncsmc} corresponds to that of Mr Bayes. The bottom clade partitions monkeys, while the central and top clades partition
hominids and prosimians. MrBayes uses 20,000 iterations of \textsc{Mcmc} in contrast to \textsc{Vncsmc} which uses 256 samples.
    }
    \label{fig:primatestree}
\end{figure*}

\section{Experiments}
\label{results}
We evaluate \textsc{Vcsmc} and \textsc{Vncsmc} on three tasks: (i) a standard dataset of primate mitochondrial DNA, (ii) on the complete 36 kilobase genomes of 17 species of Betacoronavirus, and (iii) on 7 large taxa benchmarks datasets ranging from 27 to 64 taxa. For experiments using the same initialization of likelihood and prior, the proposed methods converge to higher log-marginal likelihood values than existing methods. Additionally, they can be more easily adopted to a variety of models, with arbitrary settings of parameters $\theta$. \textsc{Vcsmc} and \textsc{Vncsmc} also scale well with the number of sites in input sequences.

\begin{table*}[!ht]
\centering
\caption{Log-marginal likelihood estimates of different variational inference techniques across 7 benchmark datasets for Bayesian phylogenetic inference. Results reported by \textsc{Vbpi}~\citep{zhang2018variational} and \textsc{Vbpi-Nf}~\citep{zhang2020improved} were obtained by (i) using 10 replicates of 10,000 maximum likelihood bootstrap trees ~\citep{10.1093/molbev/mst024} to obtain topologies defining the support of the conditional probability tables and (ii) performing 400,000 parameter updates. \textsc{Vcsmc} does not require bootstrapped or \textsc{Mcmc} tree topologies in order to learn parameters. We give $\textsc{Vcsmc}$ 2048 particles and evaluate the likelihood after 100 parameter updates. Results for $\textsc{Vcsmc}$ and $\textsc{Vcsmc}$ (JC) are averaged over three random seeds. \textsc{Vcsmc} consistently explores higher probability phylogenies than \textsc{Vbpi} and \textsc{Vbpi-Nf} without the use of preloaded topologies.}
\vspace{2mm}

\begin{tabular} {cccccccc}
\toprule
\multicolumn{8} {  c  }{Log Marginal Likelihood}\\
Dataset  & Reference & \# Taxa $(N)$ & \# Sites $(S)$ & \textsc{Vbpi} & \textsc{Vbpi-Nf} & \textsc{Vcsmc} & \textsc{Vcsmc (JC)}  \\ 
\midrule
DS1 & \cite{10.1093/oxfordjournals.molbev.a040628} & 27 & 1949 & -7108.4  & -7108.4 & \text{-5929.8} & -6906.7 \\
DS2 & \cite{Garey1996} & 29 & 2520 & -26367.7 & -26367.7 &   \text{-14160.4} & -23252.6  \\ 
DS3 & \cite{10.1080/10635150390235557} & 36 & 1812 & -33735.1 & -33735.1  &  \text{-17460.5} & -33177.8  \\
DS4 & \cite{doi:10.1080/15572536.2004.11833059} & 41 & 1137 & -13329.9 & -13329.9 & \text{-11251.9} & -12232.6 \\
DS5 & \cite{lakner} & 50 & 378 & -8214.5  & -8214.5 & \text{-5797.1} & -7921.2 \\
DS6 & \cite{doi:10.1080/00275514.2001.12063167} & 50 & 1133 & -6724.3 & -6724.3 &  \text{-5216.5} & -6575.5 \\
DS7 & \cite{doi:10.1080/00275514.2001.12061283} & 64 & 1008 & -8650.6 & -8650.4 & \text{-5847.5} & -6781.5 \\
\bottomrule
\end{tabular}
\label{table:ds1to8}
\end{table*}

\paragraph{Primate Mitochondrial DNA.}

We evaluate \textsc{Vcsmc} on a benchmark dataset of nucleotide sequences of homologous fragments of primate mitochondrial DNA~\citep{10.1093/oxfordjournals.molbev.a040524}. The dataset consists of 12 taxa $\{S_0,\cdots, S_{11}\}$ over 898 sites admitting 13,749,310,575 distinct tree topologies. The set of taxa includes five species of homonoids, four species of old world monkeys, one species of new world monkey and two species of prosimians. \textsc{Vcsmc} is run with $K=\{4,8,16,32,64,128\}$ particles, whereas \textsc{Vncsmc} is run with $K = \{4,8,16,32,64,128\}$ and $M=1$ particles, each averaged over 5 random seeds. Fig. \ref{fig:primatesloglik} shows higher values of $K$ produce larger log-marginal likelihood values (tighter ELBO values) with lower stochastic gradient noise. $\textsc{Vcsmc}$ (blue) with $K \geq 16$ outperforms probabilistic path Hamiltonain Monte Carlo ($ppHMC$) shown (green trace) for comparison.  $\textsc{Vncsmc}$ (red) requires fewer epochs than $\textsc{Vcsmc}$ to converge and produces tighter ELBO / larger log-marginal likelihood values with lower stochastic gadient noise. \textsc{Vncsmc} with $(K,M)= (4,1)$ (top left) outperforms both $ppHMC$ and $\textsc{Vcsmc}$ with $ K = 128$ (bottom right).

Fig.~\ref{fig:primatestree} provides a single maximum likelihood phylogeny selected from a run of \textsc{Vncsmc} using $K,M=(256,1)$ particles, along with a phylogeny from Mr Bayes on the same dataset. The topology returned by \textsc{Vncsmc} corresponds to that of Mr Bayes. The bottom clade partitions monkeys, while central and top clades partition
hominids and prosimians respectively. 


\paragraph{Betacoronavirus Data.} The evolutionary origin of \textsc{Sars-CoV-II} and the development of its genetic variants is an open question of paramount importance in both virology and in public health. At a high level, the species \textsc{Sars-CoV-II} belongs to the genera of betacoronaviruses, which include \textsc{Oc43} and \textsc{Hku1} (which cause the common cold) of lineage A, \textsc{Sars-CoV} and \textsc{Sars-CoV-II} (which causes the disease \textsc{Covid-19}) of lineage B, and \textsc{Mers-CoV-II} (which causes the disease \textsc{Mers}) of lineage C~\citep{abf4c405af0c4229bee83cefc6b9501f}. The exact origin of \textsc{Sars-CoV-II} however is unknown; different approaches to phylogenetic inference produce statistically incompatible results~\citep{10.1093/molbev/msaa316}. Coronaviruses have relatively large genomes ranging from 26-32 kilobases, and performing analyses on the full genomes is often a challenge. Recently, it has been argued that viral recombination in betacoronaviruses often encompasses the receptor binding domain (\textsc{Rbd}) of the spike gene~\citep{Patino-Galindo2020.02.10.942748}. This process is thought to have produced a recombination event at least 11 years ago in an ancestor of \textsc{Sars-CoV-II}~\citep{Patino-Galindo2020.02.10.942748}. We use \textsc{Vncsmc} to analyze the complete genomes for 17 species of Betacoronavirus downloaded from the NCBI Viral Genomes Resource~\citep{10.1093/nar/gku1207}. Multiple Sequence Aligmnent using Clustal was performed and each nucleotide was one-hot encoded as a vector, producing input sequences with 36,889 sites. Fig. \ref{fig:betacoronavirus} of the Appendix provides the maximum likelihood phylogeny from a \textsc{Vncsmc} run using $K,M=(256,1)$ particles. The result shows that the phylogeny partitions four lineages into clades: Embecovirus (\textit{lineage A}), Sarbecovirus (\textit{lineage B} including \textsc{Sars-CoV} and \textsc{Sars-CoV-II}), Merbecovirus (\textit{lineage C}), and Nobecovirus (\textit{lineage D}) \citep{clwppgprgpkn13, whly10, 10.1371/journal.pone.0194527}.

\paragraph{Large Taxa Benchmarks.} 
We evaluate $\textsc{Vcsmc}$ on 7 large benchmark datasets for Bayesian phylogenetic inference \citep{10.1093/oxfordjournals.molbev.a040628, Garey1996, 10.1080/10635150390235557, doi:10.1080/15572536.2004.11833059, lakner, doi:10.1080/00275514.2001.12063167,  doi:10.1080/00275514.2001.12061283}. Each dataset ranges from 27 to 64 eukaryote species with 378 to 2520 sites. Table \ref{table:ds1to8} provides the marginal likelihood values for various methods. \textsc{Vbpi}~\citep{zhang2018variational} and \textsc{Vbpi-Nf}~\citep{zhang2020improved} both learn a simplified model of molecular evolution referred to as Jukes-Cantor (JC), which fixes the transition matrix~\citep{JUKES196921}. For a fair comparison, we report \textsc{Vcsmc} (JC) results in addition to the harder task of also learning the transition matrix. Results reported by \textsc{Vbpi} and \textsc{Vbpi-Nf} were obtained by (i) using 10 replicates of 10,000 maximum likelihood bootstrap trees ~\citep{10.1093/molbev/mst024} to obtain topologies defining the support of the conditional probability tables and (ii) performing 400,000 parameter updates. Without bootstrap trees, the conditional probability tables for \textsc{Vbpi} and \textsc{Vbpi-Nf} scale exponentially with the number of taxa~\citep{zhang2018variational}. \textsc{Vcsmc} does not restrict the support of the tree topologies and instead leverages \textsc{Csmc} to compute a lower bound.  We give $\textsc{Vcsmc}$ 2048 particles and evaluate the likelihood after 100 parameter updates, averaged over three random seeds. Both \textsc{Vcsmc} and \textsc{Vcsmc} (JC) explore higher probability spaces than \textsc{Vbpi} and \textsc{Vbpi-Nf}.

\paragraph{Empirical Running Times.} We report the empirical running times of \textsc{Vcsmc} and \textsc{Vncsmc} on the primates dataset and highlight the results in Table~\ref{table:runningtimes} of the Appendix. Experiments were performed on a 2.4GHz 8-core Intel i9 processor Macbook pro with 64 GB memory and no GPU utilization. We note that alternative methods are designed for solving simpler problems in both inference and learning making any runtime comparisons indirect. For instance, \textsc{Vbpi} and \textsc{Vbpi-Nf} use precomputed topologies, while $ppHMC$ support Jukes-Cantor models.  \textsc{Vcsmc} runs on the primates dataset at an average speed of 19.34 iterations per second (\textit{it/s}) with $K$ = 4 and an average of 2.25 seconds per iteration (\textit{s/it}) with $K$ = 256. \textsc{Vncsmc} runs in 3.89 seconds per iteration with $K$ = 4 and 21.77 seconds per iteration with $K$ = 256. These numbers can be improved by leveraging GPU utilization. In contrast, MrBayes in Figure \ref{fig:primatestree} takes 12 seconds with 20,000 iterations, and the minimum it would take to converge on the primates dataset is $\sim$2,000 iterations, implying a runtime of $\sim$1.2 seconds. We observe that the first epoch of \textsc{Vcsmc} and \textsc{Vncsmc} is equivalent to the inference task and runs faster than MrBayes.

\section{Discussion}
\paragraph{Computational Complexity.} The locally optimal proposal in \textsc{Ncsmc} requires additional computational complexity to marginalize the intermediate target densities in exchange for a more informed exploration of partial states. \textsc{Ncsmc} costs $\mathcal{O}(KN^3M)$ in contrast to $\mathcal{O}(KNM)$ for \textsc{Csmc}. Empirically, \textsc{Ncsmc} with small $K,M$ produces a more accurate posterior approximation than \textsc{Csmc} with larger $K$ (see Fig.~\ref{fig:primatesloglik}). \textsc{Ncsmc} can accommodate a large number of particles with low memory overhead, however maintaining the computational graph and applying the sum-product algorithm symbolically for each of the $K$ samples and $M$ sub-samples, along with evaluating gradients for each of these terms places a practical restriction on the values of $K,M$ and $N$ used with \textsc{Vncsmc} without GPU utilization. Alternative implementations of Bayesian phylogenetic inference are computationally intensive. 
While the process of enumerating the ${N-r \choose 2}$ topologies across rank events cannot be avoided, we find that choosing $K$ as large as possible and $M=1$ is a useful heuristic for producing good results. For example, $K,M=(256,1)$ can be run on the betacoronavirus data with $N=17$ and 36,889 sites (see Fig.~\ref{fig:betacoronavirus}) without GPU utilization. One advantage of \textsc{Vcsmc} and \textsc{Vncsmc} is the ability to use minibatch iteration to speed up training. The experiments were
trained using \textsc{Adam} with a batch size $B = S/4$. Opportunities exist to parallelize \textsc{Vcsmc} and leverage GPU optimization which we expect would produce significant performance gains on DS1-DS7 as $K$ increases. 

\paragraph{Effective Sample Size.} One pertinent theoretical question concerns the relationship between the effective sample size (\textsc{Ess}), the number of samples $K$ and the number of taxa $N$.  The \textsc{Ess} measures the diversity among samples and is defined as $\textsc{Ess} = (\sum_i w_i)^2 / \sum_i w_i^2$ where $w_i$ are the unnormalized weights. We report \textsc{Ess} values on the primates data in Table~\ref{table:runningtimes} of the Appendix. While an \textsc{Ess} close to $K$ is not sufficient to ensure a good approximation, it is a necessary condition. We find near optimal \textsc{Ess} values across all choices of $K$ for both \textsc{Vcsmc} and \textsc{Vncsmc}. The theoretical foundations for developing lower bounds on \textsc{Ess} for a given value of $K$ have only been developed in the context of \textit{online inference}, where the posterior distribution is updated as new sequence data becomes available~\citep{10.1093/sysbio/syx087}. We leave theoretical questions of \textsc{Ess} and online extensions of \textsc{Vcsmc} for future work.

\paragraph{Contacts Outside of Phylogenetic Inference.} \textsc{Vcsmc} and \textsc{Vncsmc} may be adapted to a wide class of problems outside of phylogenetic inference. In principle, any generative model of data simulated by a Markov tree can be fit using \textsc{Vcsmc} and \textsc{Vncsmc}. Coalescent models for heirarchical Bayesian clustering and diffusion trees~\citep{teh2009bayesian,NIPS2012_c73dfe6c,10.1214/10-AOAS435} are examples of alternative probabilistic approaches involving distributions over latent trees that may be suited for \textsc{Vcsmc}. The nested \textsc{Csmc} algorithm may also be used to simulate approximate solutions to other combinatorial optimization tasks. Combinatorial Monte Carlo methods are used to approximate the number of self-avoiding random walks on the lattice~\citep{SOKAL1996172,SHIRAI2013}. Another point of contact is the reconstruction of jet structures in particle physics for the analysis of data from experiments at the Large Hadron Collider at CERN~\citep{hche2014introduction}. Jet reconstruction algorithms are typically based on greedy approximation methods~\citep{Cacciari2008,Dokshitzer1997}, however \textsc{Vcsmc} and \textsc{Vncsmc} may be particularly suited for this domain. The aforementioned extensions are open directions for further development.

\paragraph{Conclusion.} 
We have introduced \textsc{Vcsmc}, a powerful framework for both inference and learning in Bayesian phylogenetics. \textsc{Vcsmc} is the first method to establish the use of variational sequential search to learn distributions over intricate combinatorial structures, uncovering connections between \textsc{Vi} and \textsc{Smc}.
We have introduced \textsc{Ncsmc}, and proved that it provides an exact approximation to the locally optimal proposal for \textsc{Csmc}.  We have used \textsc{Ncsmc} to define a second objective, \textsc{Vncsmc} which yields tighter lower bounds than \textsc{Vcsmc}. \textsc{Vcsmc} and \textsc{Vncsmc} outperform existing methods on a range of tasks. A TensorFlow implementation of both \textsc{Vcsmc} and \textsc{Vncsmc} is available online at \url{https://github.com/amoretti86/phylo}.

\begin{acknowledgements} 
We thank the reviewers for their helpful feedback. We acknowledge funding from NIH/NCI grant U54CA209997 and two NIH shared instrumentation grants, S10 OD012351 and S10 OD021764. This work is also supported by ONR N00014-17-1-2131, ONR N00014-15-1-2209, DARPA SD2 FA8750-18-C-0130, Amazon, Sloan Foundation, and the Simons Foundation.
\end{acknowledgements}

\bibliography{uai2021-template}

\onecolumn
\section*{Appendix}
\label{appendix}


\begin{algorithm}[h!]
   \caption{Combinatorial Sequential Monte Carlo}
   \label{alg:csmc}
    \begin{algorithmic}
   \STATE {\bfseries Input:} $\mathbf{Y} = \{Y_1,\cdots,Y_M \} \in \Omega^{NxM}$,  $\theta = (\mathbf{Q},\{\lambda_i\}_{i=1}^{|E|})$ 
   \end{algorithmic}
   \begin{algorithmic}[1]
   \STATE Initialization. $\forall k$, $s_{0}^k\leftarrow \perp$, $w_{0}^k\leftarrow 1/K$. 
   \FOR{$r=0$ {\bfseries to} $R=N-1$}
   \FOR{$k=1$ {\bfseries to} $K$}
   \STATE  \textsc{Resample} \[ 
    \mathbb{P}(a_{r-1}^k = i) = \frac{w_{r-1}^i}{\sum_{l=1}^K w_{r-1}^l}
   \]
   \STATE \textsc{Extend partial state} \[ s_{r}^k \sim q(\cdot|s_{r-1}^{a_{r-1}^k})
   \]
   \STATE \textsc{Compute weights} \[ 
     w_{r}^k = w(s_{r-1}^{a_{r-1}^k},s_{r}^k) = \frac{\pi(s_{r}^k)}{\pi(s_{r-1}^{a_{r-1}^k})}\cdot \frac{\nu^{-}(s_{r-1}^{a_{r-1}^k})}{q(s_{r}^k |s_{r-1}^{a_{r-1}^k})}
   \]
   \ENDFOR
   \ENDFOR
   \STATE {\bfseries Output:} $s_{R}^{1:K}$ , $w_{1:R}^{1:K}$
    \end{algorithmic}
    
\end{algorithm}

The proposal distribution for \textsc{Csmc} and approximate posterior for \textsc{Vcsmc} can be written explicitly as follows:
\begin{equation}
    Q_{\phi,\psi}\left(\mathcal{T}_{1:R}^{1:K},\mathcal{B}_{1:R}^{1:K}, a_{1:R-1}^{1:K}\right) \coloneqq
    \left(\prod\limits_{k=1}^{K}q_{\phi}(\mathcal{T}_{1}^{k})\cdot q_{\psi}(\mathcal{B}_{1}^{k}) \right) \cdot 
    \prod\limits_{r=2}^{R}\prod\limits_{k=1}^{K}\left[ \frac{w_{r-1}^{a_{r-1}^k}}{\sum_{l=1}^K w_{r-1}^l}\cdot
    q_{\phi}\left(\mathcal{T}_{r}^{k}|\mathcal{T}_{r-1}^{a_{r-1}^k}\right)\cdot q_{\psi}\left(\mathcal{B}_{r}^{k}|\mathcal{B}_{r-1}^{a_{r-1}^k},\mathcal{T}_{r-1}^{a_{r-1}^k}\right)\right] 
    .
    \label{eq:fullposterior}
\end{equation}
State $s_{r}^k = (\mathcal{T}_{r}^k,\mathcal{B}_r^{k})$ is sampled by proposing forest $\mathcal{T}_{r}^{k} \sim q_{\phi}(\cdot |\mathcal{T}_{r-1}^{a_{r-1}^k})$ and branch lengths $\mathcal{B}_{r}^{k}\sim q_{\psi}(\cdot|\mathcal{B}_{r-1}^{a_{r-1}^k},\mathcal{T}_{r-1}^{a_{r-1}^k})$ from \textsc{Uniform} and \textsc{Exponential} distributions corresponding to Eq. \ref{eq:posterior} with $\phi$ and $\psi$ denoting discrete and continuous terms.

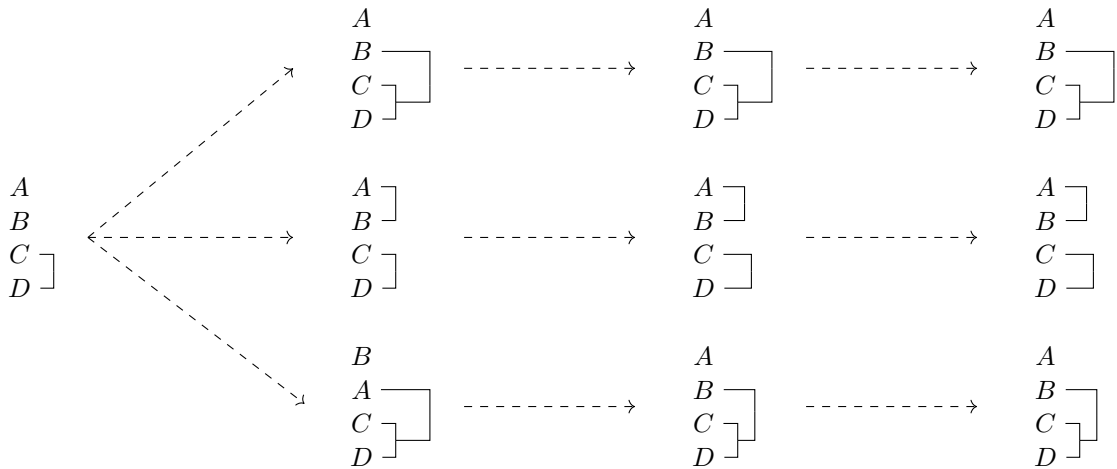
\begin{figure*}[hb!]
    \centering
    \begin{tikzpicture}[sloped, scale=0.9]
    \tikzset{
    dep/.style={circle,minimum size=1,fill=orange!20,draw=orange,
                general shadow={fill=gray!60,shadow xshift=1pt,shadow yshift=-1pt}},
    dep/.default=1cm,
    cli/.style={circle,minimum size=1,fill=white,draw,
                general shadow={fill=gray!60,shadow xshift=1pt,shadow yshift=-1pt}},
    cli/.default=1cm,
    obs/.style={circle,minimum size=1,fill=gray!20,draw,
                general shadow={fill=gray!60,shadow xshift=1pt,shadow yshift=-1pt}},
    obs/.default=1cm,
    spl/.style={cli=1,append after command={
                  node[circle,draw,dotted,
                       minimum size=1.5cm] at (\tikzlastnode.center) {}}},
    spl/.default=1cm,
    c1/.style={-stealth,very thick,red!80!black},
    v2/.style={-stealth,very thick,yellow!65!black},
    v4/.style={-stealth,very thick,purple!70!black}}
    
        \node(Resample) at (-4.5,6) {\textsc{Enumerate Topologies}};
        
        
        
        \node (E) at (-7,2.5) {$A$};
        \node (F) at (-7,2) {$B$};
        \node (G) at (-7,1.5) {$C$};
        \node (H) at (-7,1) {$D$};
        \node (gh) at (-6.5,1.25) {};
        \draw (G) -| (gh.center) {};
        \draw (H) -| (gh.center) {};
        
        \coordinate (bone) at (-6, 1.75);
        
        
        \coordinate (cone) at (-6, -.75);
    
        \node(Propose) at (.75,6) {\textsc{Subsample Branch Lengths}};
        
        \node (AA) at (-2,5) {$A$};
        \node (BB) at (-2,4.5) {$B$};
        \node (CC) at (-2,4) {$C$};
        \node (DD) at (-2,3.5) {$D$};
        \node (ccdd) at (-1.5,3.75) {};
        \draw (CC) -| (ccdd.center);
        \draw (DD) -| (ccdd.center);
        \node (bccdd) at (-1.,4) {};
        \draw (ccdd.center) -| (bccdd.center);
        \draw (BB) -| (bccdd.center);
        
        \coordinate (atwo) at (-3, 4.25);
        \coordinate (athree) at (-.5, 4.25);

        \node (EE) at (-2,2.5) {$A$};
        \node (FF) at (-2,2) {$B$};
        \node (GG) at (-2,1.5) {$C$};
        \node (HH) at (-2,1) {$D$};
        \node (gghh) at (-1.5,1.25) {};
        \draw (GG) -| (gghh.center);
        \draw (HH) -| (gghh.center);
        \node (eeff) at (-1.5,2) {};
        \draw (EE) -| (eeff.center);
        \draw (FF) -| (eeff.center);

        \coordinate (btwo) at (-3, 1.75);
        \coordinate (bthree) at (-.5, 1.75);
        \draw[->,dashed] (bone) -- (btwo);
        \draw[->,dashed] (bone) -- (atwo);
        \draw[->,dashed] (bone) -- (ctwo);
        
        \node (II) at (-2,0) {$B$};
        \node (JJ) at (-2,-.5) {$A$};
        \node (KK) at (-2,-1) {$C$};
        \node (LL) at (-2,-1.5) {$D$};
        \node (kkll) at (-1.5, -1.25) {};
        \draw (KK) -| (kkll.center);
        \draw (LL) -| (kkll.center);
        \node (ikl) at (-1., -.875) {};
        \draw (kkll.center) -| (ikl.center);
        \draw (JJ) -| (ikl.center);
        
        \coordinate (ctwo) at (-3, -.75);
        \coordinate (cthree) at (-.5, -.75);

        \node (U) at (3,5) {$A$};
        \node (V) at (3,4.5) {$B$};
        \node (W) at (3,4) {$C$};
        \node (X) at (3,3.5) {$D$};
        \node (wwxx) at (3.5, 3.75) {};
        \draw (W) -| (wwxx.center);
        \draw (X) -| (wwxx.center);
        \node (vwwxx) at (4, 4.25) {};
        \draw (wwxx.center) -| (vwwxx.center);
        \draw (V) -| (vwwxx.center);
        \coordinate (afour) at (2,4.25);
        \draw[->, dashed] (athree) -- (afour);
        
        \node (M) at (3,2.5) {$A$};
        \node (N) at (3,2) {$B$};
        \node (O) at (3,1.5) {$C$};
        \node (P) at (3,1) {$D$};
        \node (MN) at (3.6, 2.25) {};
        \draw (M) -| (MN.center);
        \draw (N) -| (MN.center);
        \node (OP) at (3.7, 1.25) {};
        \draw (O) -| (OP.center);
        \draw (P) -| (OP.center);
        \coordinate (bfour) at (2,1.75);
        \draw[->, dashed] (bthree) -- (bfour);

        \node (Q) at (3,0.) {$A$};
        \node (R) at (3,-.5) {$B$};
        \node (S) at (3,-1.) {$C$};
        \node (T) at (3,-1.5) {$D$};
        \node (ST) at (3.5, -1.25) {};
        \draw (S) -| (ST.center);
        \draw (T) -| (ST.center);
        \node (RST) at (3.75,-.825) {};
        \draw (ST.center) -| (RST.center);
        \draw (R) -| (RST.center);
        
        \coordinate (cfour) at (2,-.75);
        \draw[->, dashed] (cthree) -- (cfour);

        \node(Weighting) at (5.75,6) {\textsc{Compute Potentials}};
        \node (Uu) at (8,5) {$A$};
        \node (Vv) at (8,4.5) {$B$};
        \node (Ww) at (8,4) {$C$};
        \node (Xx) at (8,3.5) {$D$};
        \node (wwxx) at (8.5, 3.75) {};
        \draw (Ww) -| (wwxx.center);
        \draw (Xx) -| (wwxx.center);
        \node (vwwxx) at (9., 4.25) {};
        \draw (Vv) -| (vwwxx.center);
        \draw (wwxx.center) -| (vwwxx.center);
        
        \coordinate (afive) at (4.5,4.25);
        \coordinate (asix) at (7,4.25);
         \draw[->, dashed] (afive) -- (asix);
        
        \node (Mm) at (8,2.5) {$A$};
        \node (Nn) at (8,2) {$B$};
        \node (Oo) at (8,1.5) {$C$};
        \node (Pp) at (8,1) {$D$};
        \node (MN) at (8.6, 2.25) {};
        \draw (Mm) -| (MN.center);
        \draw (Nn) -| (MN.center);
        \node (OP) at (8.7, 1.25) {};
        \draw (Oo) -| (OP.center);
        \draw (Pp) -| (OP.center);
        
        \coordinate (bfive) at (4.5,1.75);
        \coordinate (bsix) at (7,1.75);
        \draw[->, dashed] (bfive) -- (bsix);
        
        \node (Qq) at (8,0.) {$A$};
        \node (Rr) at (8,-.5) {$B$};
        \node (Ss) at (8,-1.) {$C$};
        \node (Tt) at (8,-1.5) {$D$};
        \node (ST) at (8.5, -1.25) {};
        \draw (Tt) -| (ST.center);
        \draw (Ss) -| (ST.center);
        \node (RST) at (8.75, -.875) {};
        \draw (ST.center) -| (RST.center);
        \draw (Rr) -| (RST.center);
        
        \coordinate (cfive) at (4.5,-.75);
        \coordinate (csix) at (7,-.75);
        \draw[->, dashed] (cfive) -- (csix);

    \end{tikzpicture}
    \caption{Overview of the \textsc{Ncsmc} framework. 
    The enumerated topologies for state  $\{A,B,\{C,D\}\}$ are  (\textit{top}): $\{A,\{B,\{C,D\}\}\}$, (\textit{center}): $\{\{A,B\},\{C,D\}\}$  and (\textit{bottom}): $\{B,\{A,\{C,D\}\}\}$ . $M=1$ \textit{sub-branch} lengths are sampled for each edge. \textit{Sub-weights} or \textit{potentials} are computed (right). A single candidate is sampled to form the new partial state. }
    \label{fig:ncsmc}
\end{figure*}

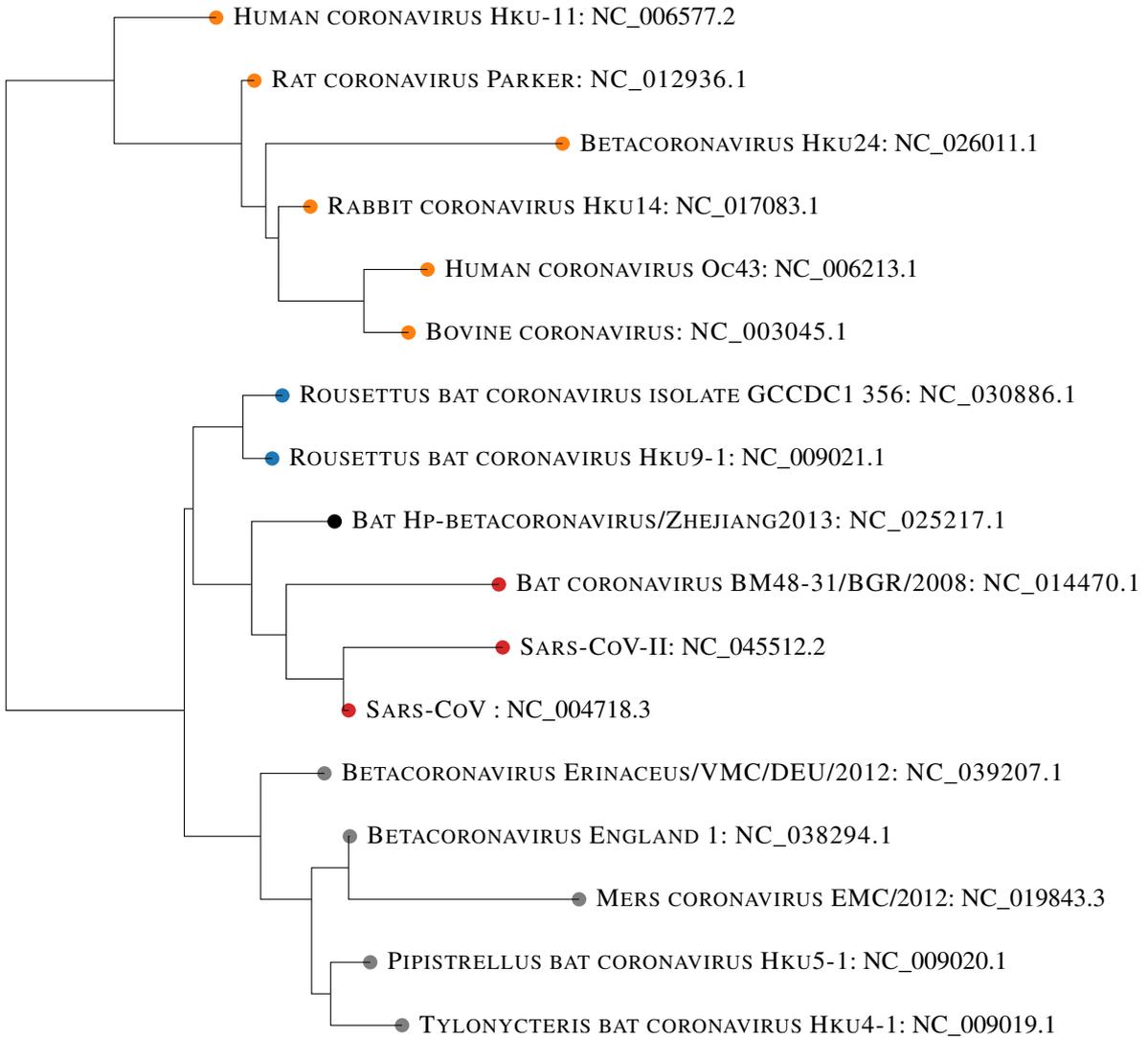
\begin{figure*}[!ht]  
\centering  
\begin{tikzpicture}[rotate=90,scale=1.75]

\node (s15)[circle,fill={rgb,255:red,255; green,127; blue,14},inner sep=2pt,label=right:$\textsc{Bovine coronavirus: NC\_003045.1}$] at (3.5,1.84) {}; 
\node (s16)[circle,fill={rgb,255:red,255; green,127; blue,14},inner sep=2pt,label=right:$\textsc{Human coronavirus }\textsc{Oc}\text{43: NC\_006213.1}$] at (4,1.69) {}; 
\node (02) at (3.75,2.19) {}; 
\node (s12)[circle,fill={rgb,255:red,255; green,127; blue,14},inner sep=2pt,label=right:$\textsc{Rabbit coronavirus }\textsc{Hku14}\text{: NC\_017083.1}$] at (4.5,2.61) {}; 
\node (03) at (4.25,2.86) {}; 
\node (s13)[circle,fill={rgb,255:red,255; green,127; blue,14},inner sep=2pt,label=right:$\textsc{Betacoronavirus }\textsc{Hku24}\text{: NC\_026011.1}$] at (5,0.63) {}; 
\node (06) at (4.5,2.96) {}; 
\node (s11)[circle,fill={rgb,255:red,255; green,127; blue,14},inner sep=2pt,label=right:$\textsc{Rat coronavirus Parker: NC\_012936.1}$] at (5.5,3.05) {}; 
\node (07) at (5,3.15) {}; 
\node (s14)[circle,fill={rgb,255:red,255; green,127; blue,14},inner sep=2pt,label=right:$\textsc{Human coronavirus }\textsc{Hku-1}\text{1: NC\_006577.2}$] at (6,3.35) {}; 

\node (s6)[circle,fill={rgb,255:red,214; green,39; blue,40},inner sep=2pt,label=right:$\textsc{Sars-CoV} \text{ : NC\_004718.3}$] at (0.5,2.31) {}; 
\node (s8)[circle,fill={rgb,255:red,214; green,39; blue,40},inner sep=2pt,label=right:$\textsc{Sars-CoV-II} \text{: NC\_045512.2}$] at (1,1.1) {}; 

\node (04) at (0.75,2.35) {}; 
\node (s9)[circle,fill={rgb,255:red,214; green,39; blue,40},inner sep=2pt,label=right:$\textsc{Bat coronavirus BM48-31/BGR/2008: NC\_014470.1}$] at (1.5,1.13) {}; 

\node (05) at (1.1,2.8) {}; 
\node (s10)[circle,fill,inner sep=2pt,label=right:$\textsc{Bat Hp-betacoronavirus/Zhejiang2013: NC\_025217.1}$] at (2,2.42) {}; 

\node (s4)[circle,fill={rgb,255:red,31; green,119; blue,180},inner sep=2pt,label=right:$\textsc{Rousettus bat coronavirus }\textsc{Hku9-1}\text{: NC\_009021.1}$] at (2.5,2.91) {}; 
\node (s5)[circle,fill={rgb,255:red,31; green,119; blue,180},inner sep=2pt,label=right:$\textsc{Rousettus bat coronavirus isolate GCCDC1 356: NC\_030886.1}$] at (3,2.83) {}; 

\node (13) at (1.5,3.07) {}; 
\node (12) at (2.75,3.14) {}; 

\node (s1)[circle,fill={rgb,255:red,127; green,127; blue,127},inner sep=2pt,label=right:$\textsc{Tylonycteris bat coronavirus }\textsc{Hku4-1}\text{: NC\_009019.1}$] at (-2,1.89) {}; 
\node (s7)[circle,fill={rgb,255:red,127; green,127; blue,127},inner sep=2pt,label=right:$\textsc{Pipistrellus bat coronavirus }\textsc{Hku5-1}\text{: NC\_009020.1}$] at (-1.5,2.14) {}; 
\node (s0)[circle,fill={rgb,255:red,127; green,127; blue,127},inner sep=2pt,label=right:$\textsc{Mers} \textsc{ coronavirus }\textsc{EMC}\text{/2012: NC\_019843.3}$] at (-1,0.5) {}; 
\node (s2)[circle,fill={rgb,255:red,127; green,127; blue,127},inner sep=2pt,label=right:$\textsc{Betacoronavirus England 1: NC\_038294.1}$] at (-0.5,2.3) {}; 

\node (09) at (-1.75,2.45) {}; 
\node (01) at (-0.75,2.31) {}; 

\node (10) at (-1,2.6) {}; 
\node (s3)[circle,fill={rgb,255:red,127; green,127; blue,127},inner sep=2pt,label=right:$\textsc{Betacoronavirus Erinaceus/VMC/DEU/2012: NC\_039207.1}$] at (0,2.5) {}; 

\node (11) at (-0.5,3) {}; 
\node (14) at (2.1,3.53) {}; 

\node (08) at (5.5,4.15) {}; 
\node (15) at (0.5,3.6) {}; 
\node (16) at (2,5) {}; 

\draw  (s1.center) |- (09.center);
\draw  (s7.center) |- (09.center);
\draw  (s0.center) |- (01.center);
\draw  (s2.center) |- (01.center);
\draw  (09.center) |- (10.center);
\draw  (01.center) |- (10.center);
\draw  (10.center) |- (11.center);
\draw  (s3.center) |- (11.center);
\draw  (s6.center) |- (04.center);
\draw  (s8.center) |- (04.center);
\draw  (04.center) |- (05.center);
\draw  (s9.center) |- (05.center);
\draw  (05.center) |- (13.center);
\draw  (s10.center) |- (13.center);
\draw  (s4.center) |- (12.center);
\draw  (s5.center) |- (12.center);
\draw  (13.center) |- (14.center);
\draw  (12.center) |- (14.center);
\draw  (s15.center) |- (02.center);
\draw  (s16.center) |- (02.center);
\draw  (02.center) |- (03.center);
\draw  (s12.center) |- (03.center);
\draw  (03.center) |- (06.center);
\draw  (s13.center) |- (06.center);
\draw  (06.center) |- (07.center);
\draw  (s11.center) |- (07.center);
\draw  (07.center) |- (08.center);
\draw  (s14.center) |- (08.center);
\draw  (11.center) |- (15.center);
\draw  (14.center) |- (15.center);
\draw  (15.center) |- (16.center);
\draw  (08.center) |- (16.center);

\end{tikzpicture}
\caption{Overview of the betacoronavirus results. The data consists of 17 species of betacoronavirus across 36,889 sites. \textsc{Vncsmc} is run using $K,M=(256,1)$. A single nonclock phylogeny is chosen based on maximum likelihood and displayed. Colors denote species from the four varying viral lineages: Embecovirus (orange \textit{lineage A}); Nobecovirus (blue \textit{lineage D}); Sarbecovirus (red \textit{lineage B} including \textsc{Sars-CoV} and \textsc{Sars-CoV-II}); Merbecovirus (grey \textit{lineage C}) and Hibecovirus (black \textit{not classified into the four lineages}) are each partitioned in clades.}
\label{fig:betacoronavirus}
\end{figure*}

\clearpage
\begin{figure}
\centering
\includegraphics[width=0.6\textwidth]{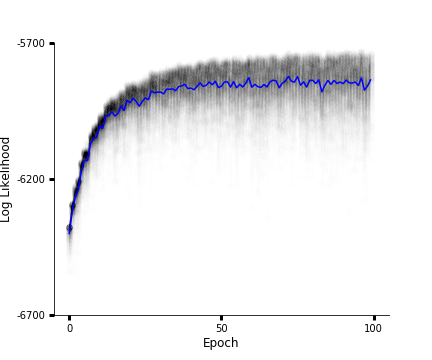}
\caption{\textsc{Vncsmc} on the primates data with $K, M$ = (128, 1). The full distribution of log likelihood values for all particles across epochs is plotted in black. The average likelihood across samples is plotted in blue. }
\label{fig:loglikallsamples}
\end{figure} 

\begin{table}
\centering

\begin{tabular} {ccccccccc}
\toprule
& &  \textsc{Vcsmc} & & & & \textsc{Vncsmc} &\\
$K$ & \textit{s/it} & \textit{s/mit} & \textit{time (minutes)} & \textsc{Ess} & \textit{s/it} & \textit{s/mit} & \textit{time (minutes)} & \textsc{Ess} \\ 
\midrule
4 &  5.17e-2 & 1.31e-2  & 0:22 & 3.98 & 4.01 & 1.17  & 6:32 & 3.99\\
8 & 5.58e-2 & 1.42e-2 & 0:28  & 7.96 &   4.27 & 1.24  & 7:09 & 7.88\\ 
16 &  3.11e-2 & 7.76e-2 & 0:30  & 15.79 &  4.83 & 1.53  & 8:15 & 15.62\\
32 &  5.78e-2 & 2.17e-1 & 0:49 & 31.72 & 5.98 & 1.59 & 10:17 & 31.00 \\
64 &  9.80e-2 & 2.66e-1  & 1:23 & 62.92 & 8.33 & 2.09 & 14:33 & 62.59 \\
128 &  1.35 & 3.48e-1 & 2:16 & 122.79 & 11.88 & 2.89 & 20:02 & 124.23 \\
256  & 2.25 & 5.95e-1 & 3:52 & 252.02 & 21.77 & 4.98 & 36:51 & 252.43 \\
\bottomrule
\end{tabular}
\caption{Empirical running times of \textsc{Vcsmc} and \textsc{Vncsmc}. The Primates data consists of 12 taxa over 898 sites admitting 13,749,310,575 distinct tree topologies. Experiments were performed on a 2.4GHz 8-core intel i9 processor Macbook Pro with 64 GB memory and no GPU utilization. We profile using $K = \{4,8,16,32,64,128,256\}$ and $M=1$. The left column provides seconds per iteration (\textit{s/it}), the left center column provides seconds per minibatch (\textit{s/mit}), the center right column provides total running time (minutes) across 100 epochs. The effective sample size is provided in the right columns.}
\label{table:runningtimes}
\end{table}

\begin{figure}[ht!]
\centering
\begin{tikzpicture}[sloped]
\node (A) at (-7.5,0) {$A$};
\node (B) at (-6.5,0) {$B$};
\node (C) at (-5.5,0) {$C$};
\node (D) at (-4.5,0) {$D$};
\node (ab) at (-7.,1) {};
\node (cd) at (-5.,1) {};
\node (abcd) at (-6.,2) {};

\draw  (A) |- (ab.center);
\draw  (B) |- (ab.center);
\draw  (C) |- (cd.center);
\draw  (D) |- (cd.center);
\draw  (cd.center) |- (abcd.center);
\draw  (ab.center) |- (abcd.center);

\node (e) at (-3.5,0) {$A$};
\node (f) at (-2.5,0) {$B$};
\node (g) at (-1.5,0) {$C$};
\node (h) at (-.5,0) {$D$};
\node (ef) at (-3.,1) {};
\node (efg) at (-2.,2) {};
\node (efgh) at (-2.,3) {};

\draw  (e) |- (ef.center);
\draw  (f) |- (ef.center);
\draw  (g) |- (efg.center);
\draw  (h) |- (efgh.center);
\draw  (ef.center) |- (efg.center);
\draw  (efg.center) |- (efgh.center);

\node (i) at (.5,0) {$A$};
\node (j) at (1.5,0) {$B$};
\node (k) at (2.5,0) {$D$};
\node (l) at (3.5,0) {$C$};
\node (ij) at (1.,1) {};
\node (ijk) at (2.,2) {};
\node (ijkl) at (3.,3) {};

\draw  (i) |- (ij.center);
\draw  (j) |- (ij.center);
\draw  (k) |- (ijk.center);
\draw  (l) |- (ijkl.center);
\draw  (ij.center) |- (ijk.center);
\draw  (ijk.center) |- (ijkl.center);

\end{tikzpicture}
\caption{Overview of the dual representation of a partial state. The partial state $s_{1}^1 = \{P_{AB},C,D\}$ for four taxa corresponding to Fig.~\ref{fig:nfe} is illustrated using its dual representation $\mathcal{D}(s)$. The dual state $\mathcal{D}(s)\subseteq \mathcal{T}$ corresponds to the three complete tree topologies.  (\text{left}): $\{ \{A,B\},\{C,D\}\}$ (\text{center}): $ \{\{A,B\}, \{A,B,C\}\}$ and (\text{right}): $ \{\{A,B\}, \{A,B,D\}\}$.}
\label{fig:dualphyhlo}
\end{figure}
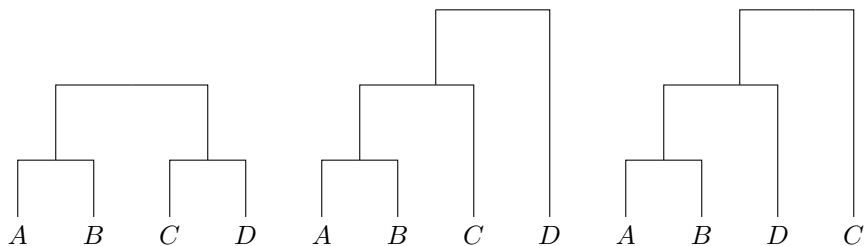

\begin{figure}
\centering
\includegraphics[width=0.6\textwidth]{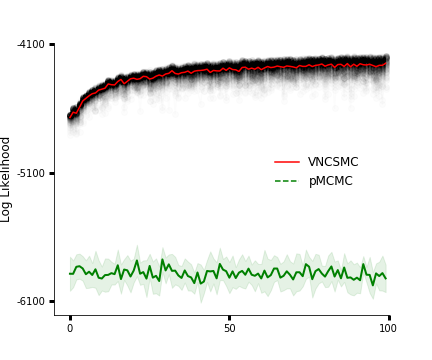}
\caption{\textsc{Vncsmc} on the 9-taxa subset of primates data with $K, M$ = (128, 1). The full distribution of log likelihood values for all \textsc{Vncsmc} particles across epochs is plotted in black. The average likelihood across samples is plotted in red. Particle Gibbs~\citep{wang2020particle} is run for 5000 iterations 10 times independently. The last 100 iterations for the 10 independent runs of Particle Gibbs are averaged and plotted in green. \textsc{Vncsmc} using 100 epochs outperforms Particle Gibbs using 5000 iterations.}
\label{fig:pMCMC}
\end{figure}

\end{document}